\documentclass[twoside]{article}

\usepackage[accepted]{aistats2024}
\usepackage[round]{natbib}

\usepackage[utf8]{inputenc} 
\usepackage[T1]{fontenc}    
\usepackage{hyperref}       
\usepackage{url}            
\usepackage{booktabs}       
\usepackage{amsfonts}       
\usepackage{nicefrac}       
\usepackage{microtype}      
\usepackage{xcolor}         
\usepackage{multirow}
\usepackage{amsmath}
\usepackage{amssymb}
\usepackage{mathtools}
\usepackage{amsthm}
\usepackage{graphicx}
\usepackage{subfigure}
\usepackage{float}
\usepackage{algorithmic}
\usepackage{algorithm,float}
\usepackage{enumerate}

\newtheorem{theorem}{Theorem}
\newtheorem{proposition}[theorem]{Proposition}

\newtheorem{assumption}[theorem]{Assumption}

\graphicspath{{./figures/}}

\begin{document}

%

%

\twocolumn[

\aistatstitle{DNNLasso: Scalable Graph Learning for Matrix-Variate Data}

\aistatsauthor{Meixia Lin \And Yangjing Zhang}

\aistatsaddress{
Singapore University of Technology and Design
\And
Chinese Academy of Sciences} ]

\begin{abstract}
  We consider the problem of jointly learning row-wise and column-wise dependencies of matrix-variate observations, which are modelled separately by two precision matrices. Due to the complicated structure of Kronecker-product precision matrices in the commonly used matrix-variate Gaussian graphical models, a sparser Kronecker-sum structure was proposed recently based on the Cartesian product of graphs. However, existing methods for estimating Kronecker-sum structured precision matrices do not scale well to large scale datasets. In this paper, we introduce DNNLasso, a diagonally non-negative graphical lasso model for estimating the Kronecker-sum structured precision matrix, which outperforms the state-of-the-art methods by a large margin in both accuracy and computational time. Our code is available at \url{https://github.com/YangjingZhang/DNNLasso}.
\end{abstract}

\section{INTRODUCTION}
In the modern big data era, matrix-variate observations (i.e., two-dimensional grids of observations) are becoming prevalent in various domains including spatial-temporal data analysis, financial markets, genomics and imaging processing. A typical example is the spatial-temporal data in weather forecasting \citep{stevens2019graph,stevens2021graph}, in which each observation contains winter precipitations of $t$ time lags (rows) and $s$ locations (columns). Due to the pervasiveness of matrix-variate observations, it is important for us to understand the structure encoded in these observations. In particular, the commonly used precision matrix (also called as the inverse covariance matrix) has received a lot of attention, as it encodes conditional independence relationships among variables.

One possible approach is to learn the precision matrices associated with the rows and columns of the matrix-variate observations separately. For example, for the spatial-temporal data in weather forecasting, we can estimate the spatial precision matrix by treating the columns of all winter precipitation observations as vector-variate spatial observations, and also estimate the temporal precision matrix by treating the rows of all winter precipitation observations as vector-variate temporal observations. In high-dimensional multivariate data analysis on vector-variate observations, many statistical models have been proposed for the estimation of the precision matrix. One widely used model is the Gaussian graphical model that learns a sparse precision matrix via an $\ell_1$-norm penalized maximum likelihood approach \citep{yuan2007model,banerjee2008model,friedman2008sparse,rothman2008sparse}. However, it is limited in our scenario since the observations in the Gaussian graphical model are assumed to be independent and identically distributed, while the vector-variate observations can be correlated in matrix-variate data analysis. For example, in the spatial-temporal data, not only the spatial observations can be correlated, but different temporal observations can also be correlated. Therefore, it is necessary to model the correlations among both rows and columns in the observations jointly.

Given matrix-variate data where each observation $Z$ is a $t\times s$ matrix, it may appear tempting to stack $Z$ as a column vector ${\rm vec}(Z)$ and model $Z$ as a $ts$-dimensional vector. Gaussian graphical models can be used to analyze the vectorized data, while they suffer from three shortcomings. First, estimating a $ts\times ts$ precision matrix can be daunting due to the extremely high dimension. Second, the analysis based on ${\rm  vec}(Z)$ ignores all row and column structural information in the observations, which is useful and sometimes vital in practice. Third, learning a precision matrix without prior structural assumptions would be impractical in high-dimension low-sample regime. Alternative approaches that explore the matrix nature of such matrix-variate observations are therefore attractive nowadays, among which the matrix-variate Gaussian graphical model is the most famous one.

The matrix-variate Gaussian distribution \citep{dawid1981some,gupta1999matrix,efron2009set,allen2010transposable,leng2012sparse} of $Z$ assumes that the covariance matrix of ${\rm vec}(Z)$ has the form of a Kronecker-product (KP) between two covariance matrices, separately associated with the rows and  columns of matrix-variate observations. The KP assumption for the covariance implies that the precision matrix is a also KP of two precision matrices, that is $\Omega\otimes \Gamma$, where $\Omega\in \mathbb{S}^s_{++}$ models the column-wise dependencies in $Z$ and $\Gamma\in \mathbb{S}^t_{++}$ models the row-wise dependencies in $Z$. This additional KP structure has been considered in many recent works by \cite{yin2012model,leng2012sparse,tsiligkaridis2013covariance,tsiligkaridis2013convergence,zhou2014gemini}, which allows one to provide a satisfying estimation of the precision matrix with a small sample size. However, the KP structure leads to a relatively dense graph and a nonconvex log-likelihood, which raises great challenges in the optimization of the model. To overcome the challenges, \cite{kalaitzis2013bigraphical} introduced a sparser structure for the precision matrix by imposing a novel Kronecker-sum (KS) structure instead of KP. For estimating the KS structured precision matrix, many optimization methods have been proposed recently by \cite{kalaitzis2013bigraphical,greenewald2019tensor,yoon2022eiglasso}.

\subsection{Related Works}
\cite{kalaitzis2013bigraphical} first considered a Gaussian distribution for a matrix variable with a novel KS structure, say $\Omega \oplus \Gamma=\Omega\otimes I_t +I_s\otimes \Gamma$, for the precision matrix. Here $I_t$ denotes a $t$ by $t$ identity matrix.  They proposed the algorithm {\tt BiGLasso}, a block coordinate descent algorithm for optimizing $\Omega$ and $\Gamma$ in the maximum likelihood estimation approach, by regarding the columns of $\Omega$ and $\Gamma$ as blocks. However, they did not tackle the non-identifiability of the diagonal entries of $\Omega$ and $\Gamma$, which is one of the key challenges in estimating the precision matrices with the KS structure. The non-identifiability arises from the fact that
$ \Omega \oplus \Gamma = (\Omega + cI_s) \oplus (\Gamma - cI_t)$ for all  $c\in\mathbb{R}$. Namely, the KS matrix $\Omega \oplus \Gamma $ does not uniquely determine the pair $(\Gamma,\Omega)$, as one can modify the diagonal entries of $\Omega$ and $\Gamma$ by adding and subtracting a constant $c$ without changing their KS.
Moreover, {\tt BiGLasso} may not scale well to median-sized datasets and the convergence of {\tt BiGLasso} was not analyzed by \cite{kalaitzis2013bigraphical}.

Later,  \cite{greenewald2019tensor} proposed a multi-way tensor generalization of the two-way KS structure for the precision  matrix studied by \cite{kalaitzis2013bigraphical}. Based on the accelerated proximal gradient method of \cite{nesterov2013gradient}, a method {\tt TeraLasso} with convergence guarantees was provided by \cite{greenewald2019tensor}. Their strategy for estimating the diagonal entries is through identifiable reparameterization with additional restrictions on the traces of $\Omega$ and $\Gamma$. Although {\tt TeraLasso} is much better than {\tt BiGLasso} in terms of convergence properties and computational speed,  {\tt TeraLasso} seems to be limited to graphs with only a few hundreds nodes.

More recently, based on a proximal Newton's method for a regularized log-determinant program \citep{hsieh2014quic}, an efficient algorithm {\tt EiGLasso} was proposed by \cite{yoon2022eiglasso} for learning the KS structured precision  matrix. They introduced a new scheme for identifying the unidentifiable diagonal entries of $\Omega$ and $\Gamma$ via introducing an additional constraint --- restricting the trace ratio of $\Omega$ and $\Gamma$ to be a fixed constant such that the KS $ \Omega \oplus \Gamma$ uniquely determines $\Omega$ and $\Gamma$. The numerical experiments by \cite{yoon2022eiglasso} show that {\tt EiGLasso} empirically has two to three orders-of-magnitude speedup compared to {\tt TeraLasso}, while it still takes hours on datasets with graph size of around one thousand.

\subsection{Contributions}
Our main contributions are summarized in four parts. First, we propose the diagonally non-negative graphical lasso ({\tt DNNLasso}) algorithm for estimating the KS precision matrix. These additional non-negative constraints on the diagonal entries of the two precision matrices $\Omega$ and $\Gamma$ naturally avoid the non-identifiability issue. Second, we develop an efficient and robust algorithm based on the alternating direction method of multipliers for solving the optimization problem in {\tt DNNLasso}, where the computational cost and memory cost are both extremely low. Third, as a key ingredient in {\tt DNNLasso}, we deduce the explicit solution of the proximal operator associated with the negative log-determinant of KS. As far as our knowledge goes, it is the first time that the explicit formula is provided. Last, numerical experiments on both synthetic data and real data demonstrate that {\tt DNNLasso} outperforms the state-of-the-art {\tt TeraLasso} and {\tt EiGLasso} by a large margin.





\subsection{Notation}
$\mathbb{R}^{m\times n}$ denotes the space of $m$ by $n$ matrices and $\mathbb{S}^n$ denotes the space of $n$ by $n$ symmetric matrices. $\mathbb{S}^n_+$ (resp. $\mathbb{S}^n_{++}$) denotes the space of $n$ by $n$ positive semidefinite (resp. definite) matrices. $\lambda_{\rm min}(X)$ denotes the smallest eigenvalue of a symmetric matrix $X$. ${\rm diag}(X)$ denotes the column vector containing the diagonal elements of the matrix $X$. The log-determinant function $\log |X| := \log \det (X)$ takes the logarithm of the determinant of the positive definite matrix $X$. $[n]:=\{1,2,\dots,n\}$. $\delta_C(\cdot)$ denotes the indicator function of the set $C$, i.e., $\delta_C(x) = 0$ if $x\in C$; $\delta_C(x) = +\infty$ if $x\notin C$. For any $X\in \mathbb{R}^{n\times n}$, $\|X\|_{1,{\rm off}} := \sum_{1\leq i\neq j\leq n} |X_{ij}|$. $I_t$ denotes a $t$ by $t$ identity matrix. The Kronecker-sum of matrices $\Gamma\in \mathbb{S}^t$ and $\Omega\in \mathbb{S}^s$ is $\Omega \oplus \Gamma:=\Omega\otimes I_t +I_s\otimes \Gamma$.

\section{ESTIMATION OF A KRONECKER-SUM PRECISION MATRIX}
Let ${\cal G}= ({\cal V},{\cal E})$ be an undirected graph with a vertex set ${\cal V}=\{1,\ldots,ts\}$ and an edge set ${\cal E}$. Each variable (e.g., a certain feature at a particular time and location in spatial-temporal data) is associated with one vertex. A random vector $z\in \mathbb{R}^{ts}$ is said to satisfy the Gaussian graphical model with graph ${\cal G}$, if $z \sim {\cal N}(0,\Sigma)$ is Gaussian 
with $(\Sigma^{-1})_{ij}=0$ for all $(i,j)\notin {\cal E}$.

Given i.i.d. observations $Z^{(1)},\ldots,Z^{(n)}$ in $\mathbb{R}^{t\times s}$ such that ${\rm vec}(Z^{(k)})\sim {\cal N}(0,\Sigma)$ for each $k$, the sparse Gaussian graphical lasso estimator \citep{yuan2007model,banerjee2008model,friedman2008sparse,rothman2008sparse} for the precision matrix $\Sigma^{-1}$ is given by
\begin{align}
	\widehat{X}\in \underset{X\in \mathbb{S}_{++}^{ts}}{\arg\min}  \Big\{-\log|X| + \langle C,X\rangle + \lambda_0 \|X\|_{1,{\rm off}}\Big\}.\label{eq: graphical_lasso}
\end{align}
Here $\lambda_0>0$ is a parameter that controls the strength of the penalty, and the sample covariance matrix is $ C =\frac{1}{n}\sum_{k=1}^n {\rm vec}(Z^{(k)}-\bar{Z})({\rm vec}(Z^{(k)}-\bar{Z}))^T\in \mathbb{S}_{+}^{ts}$ with $\bar{Z}=\sum_{k=1}^n Z^{(k)}/n$.
The $\ell_1$ regularizer $\|X\|_{1,{\rm off}}$ is added to get a sparse network, as the sparsity pattern of $\Sigma^{-1}$ determines the conditional independence structure of the $ts$ variables \citep{lauritzen1996graphical}.

When $n\ll ts$, in particular, $n=1$, the sample covariance matrix $C$ is a highly uncertain estimate of the truth $\Sigma$. More priors are required to obtain a satisfying estimation of the precision matrix $\Sigma^{-1}$.

\subsection{Kronecker-Sum Structured Precision Matrix}
A fundamental assumption in the KS model by \cite{kalaitzis2013bigraphical} is that $\Sigma^{-1}$ takes the KS form of $\Sigma^{-1} = \Omega\oplus \Gamma$. Then the problem of estimating $\Sigma^{-1}$ reduces to estimating $\Gamma$ and $\Omega$, which correspond to the row-wise and column-wise precision matrices in $Z$, respectively.
Specifically, the sparse Gaussian graphical lasso estimator in \eqref{eq: graphical_lasso} reduces to $ \widehat{X} = \widehat{\Omega}\oplus \widehat{\Gamma}$, where $(\widehat{\Gamma},\widehat{\Omega})$ is an optimal solution to the problem
\begin{align}
	\min_{\substack{\Gamma\in \mathbb{S}^{t}_{++},\\ \Omega\in \mathbb{S}^{s}_{++}}} \ \left\{
	\begin{aligned}
		&-\log|\Omega \oplus \Gamma| + \langle \Omega, W\rangle + \langle \Gamma, R\rangle  \\
		&+\lambda_0 s \|\Gamma\|_{1,{\rm off}} +\lambda_0 t \|\Omega\|_{1,{\rm off}}
	\end{aligned}\right\}.\label{eq:KS}
\end{align}
The sample row-wise and column-wise covariance matrices are
$R = \frac{1}{n}\sum_{k=1}^n Z^{(k)} (Z^{(k)})^T$ and $W = \frac{1}{n}\sum_{k=1}^n (Z^{(k)})^T Z^{(k)}$.
Throughout the paper, we make the following Assumption~\ref{assu}.

\begin{assumption}\label{assu}
    $R_{ii} >0,\,W_{jj} >0,\,\forall\,i\in[t],\,j\in [s]$.
\end{assumption}
We take $R$ as an example to illustrate the assumption. Suppose $R_{ii}=0$ for some $i\in [t]$. This implies that the $i$-th row of $Z^{(k)}$ equals to zero for every $k\in[n]$.
Consequently,  we can remove the $i$-th rows of all matrix-variate observations prior to model construction, due to their redundant nature.


\subsection{Equivalent Formulation of \eqref{eq:KS}}

We are going to construct an optimal solution to the problem \eqref{eq:KS} through solving a simpler model without the positive definite constraints on $\Gamma$ and $\Omega$, since the positive definite constraints will usually raise computational challenges in designing and implementing optimization algorithms.

The basic idea is to control the nonnegativity of diagonal elements instead of forcing the positive definiteness of $\Gamma$ and $\Omega$. Specifically, we propose the diagonally non-negative graphical lasso ({\tt DNNLasso}) model for estimating sparse row-wise and column-wise precision matrices simultaneously as
\begin{align}
	\min_{\Gamma\in \mathbb{S}^{t},\Omega\in \mathbb{S}^{s}} \ &\left\{
	\begin{aligned}
		&-\log|\Omega \oplus \Gamma| + \langle \Omega, W\rangle + \langle \Gamma, R\rangle\\
		&+\lambda_0 s \|\Gamma\|_{1,{\rm off}} +\lambda_0 t \|\Omega\|_{1,{\rm off}}
	\end{aligned}
	\right\} \label{eq:KS_NN}\\
	{\rm s.t.} \quad & \,\, \Omega \oplus \Gamma \in \mathbb{S}^{ts}_{++},\ {\rm diag}(\Omega)\geq 0 ,\  {\rm diag}(\Gamma)\geq 0,\notag
\end{align}
The following proposition formally states the equivalence between problems \eqref{eq:KS} and \eqref{eq:KS_NN}. The detailed proof can be found in Appendix~\ref{sec:prof_of_equivalence}.

\begin{proposition}\label{prop: equivalence}
Problems \eqref{eq:KS} and \eqref{eq:KS_NN} are equivalent in the following sense:

(a) they share the same optimal objective function value;

(b) any optimal solution to \eqref{eq:KS} is optimal to \eqref{eq:KS_NN};

(c) if $(\Gamma^*,\Omega^*)$ is an optimal solution to \eqref{eq:KS_NN}, then 
    \begin{align}
    &	(\widehat{\Gamma},\widehat{\Omega}):= \notag \\
&\begin{cases}
		(\Gamma^*,\Omega^*) & \mbox{if } \Gamma^*\in \mathbb{S}^t_{++}, \Omega^* \in \mathbb{S}^s_{++},\\
		(\Gamma^*-  c I_t,\Omega^* + c I_s) & \mbox{otherwise},
	\end{cases} \label{reconstruct}
    \end{align}
with $c= (\lambda_{\rm min}(\Gamma^*)-\lambda_{\rm min}(\Omega^*))/2$,
is an optimal solution to \eqref{eq:KS}.
\end{proposition}

Furthermore, prior to designing an algorithm for solving \eqref{eq:KS_NN}, we present the subsequent theorem which characterizes the solution set of \eqref{eq:KS_NN}. The proof  can be found in Appendix~\ref{sec:prof_of_sol}.
\begin{theorem}\label{thm: sol}
Under Assumption~\ref{assu}, the problem \eqref{eq:KS_NN} admits a non-empty and bounded solution set.
\end{theorem}

Due to the non-identifiability issue, for an optimal solution $(\Gamma,\Omega)$ to \eqref{eq:KS_NN} without imposing non-negativity constraints, their diagonal entries $\Gamma_{ii}$ and $\Omega_{jj}$ can possibly be extremely large values. This may cause numerical instability in optimization algorithms. However, the non-negativity constraints in \eqref{eq:KS_NN} ensure a non-empty and bounded solution set, as demonstrated in Theorem~\ref{thm: sol}. This boundedness effectively overcomes the aforementioned instability in optimization algorithms.

\section{DNNLASSO}
In order to obtain the {\tt DNNLasso} estimator, we design an efficient and robust algorithm for solving  \eqref{eq:KS_NN}. We consider an equivalent and compact form of the problem
\begin{align}
	\min_{\Gamma\in \mathbb{S}^{t},\Omega\in \mathbb{S}^{s}} \
	&\left\{
	\begin{aligned}
		&-\log|\Omega \oplus \Gamma| + \langle \Omega, W\rangle + \langle \Gamma, R\rangle  \\
		&+ p(\Gamma)+q(\Omega)
	\end{aligned}\right\}\label{eq:KSv2}\\
	{\rm s.t.} \quad & \,\, \Omega \oplus \Gamma \in \mathbb{S}^{ts}_{++},\notag
\end{align}
where $p(\Gamma) = \lambda_T \|\Gamma\|_{1,{\rm off}}$ if ${\rm diag}(\Gamma)\geq 0$, and $+\infty$ otherwise; $q(\Omega) = \lambda_S \|\Omega\|_{1,{\rm off}}$ if ${\rm diag}(\Omega)\geq 0$, and $+\infty$ otherwise.
The penalty function $p$ and $q$ necessitate the non-negativity of diagonal entries and promote the sparsity of off-diagonal entries.
We can obtain the problem \eqref{eq:KS_NN} by taking $\lambda_T = \lambda_0 s$ and $\lambda_S = \lambda_0 t$.

\subsection{Alternating Direction Method of Multipliers}
The alternating direction method of multipliers (ADMM) is well-suited for the problems with separable or block separable objectives and a mix of equality and inequality constraints; see \cite{glowinski1975approximation,gabay1976dual,eckstein1992douglas}. These characteristics align perfectly with the structure of our target problem \eqref{eq:KSv2} after the introduction of auxiliary variables. In fact, by introducing auxiliary variables $\Lambda \in \mathbb{S}^{t}$, $\Theta\in \mathbb{S}^{s}$,  $\Xi \in \mathbb{S}^{s}$, we have an equivalent form of \eqref{eq:KSv2} as
\begin{align}
	\min_{\Gamma,\Lambda\in \mathbb{S}^{t},\Omega,\Theta,\Xi\in \mathbb{S}^{s}} \
	&\left\{
	\begin{aligned}
		&-\log|\Omega \oplus \Gamma| + \langle \Xi, W\rangle + \langle \Gamma, R\rangle  \\
		&+ p(\Lambda)+q(\Theta)
	\end{aligned}\right\}\notag\\
{\rm s.t.} \quad & \,\, \Omega \oplus \Gamma \in \mathbb{S}^{ts}_{++},\quad \Gamma-\Lambda = 0,\label{eq:KSv3} \\
&\,\,   \Xi-\Theta =0,\quad \Xi-\Omega = 0.\notag
\end{align}
Given $\sigma>0$, for any $(\Gamma,\Omega,\Lambda,\Theta,\Xi,X,Y,U)\in \mathbb{S}^t\times \mathbb{S}^s \times \mathbb{S}^t\times \mathbb{S}^s \times \mathbb{S}^s \times \mathbb{S}^t\times \mathbb{S}^s\times \mathbb{S}^s$, the  augmented Lagrangian function associated with the above problem is
\begin{align*}
&{\cal L}_{\sigma}(\Gamma,\Omega,\Lambda,\Theta,\Xi;X,Y,U)\\
=& -\log|\Omega \oplus \Gamma| + \langle \Xi, W\rangle  + \langle \Gamma, R\rangle  + p(\Lambda) + q(\Theta)\\
&+\delta_{\mathbb{S}_{++}^{ts}}(\Omega \oplus \Gamma)+\frac{\sigma}{2}\|\Gamma - \Lambda - \sigma^{-1}X\|_F^2 \\
&+ \frac{\sigma}{2}\|\Xi - \Theta - \sigma^{-1}Y\|_F^2 + \frac{\sigma}{2}\|\Xi - \Omega - \sigma^{-1}U\|_F^2\\
&- \frac{1}{2\sigma}\|X\|_F^2 -  \frac{1}{2\sigma}\|Y\|_F^2 -  \frac{1}{2\sigma}\|U\|_F^2.
\end{align*}
Note that the ADMM is a primal-dual method. We are going to alternatingly minimize the primal variables among the two blocks $(\Xi,\Gamma)$ and $(\Lambda,\Theta,\Omega)$, and then update the multipliers $(X,Y,U)$. See Algorithm~\ref{alg: admm} for the full algorithm.  The convergence result stated in Theorem \ref{thm:convergence} follows from the well-known convergence property for the classical 2-block ADMM; see \cite{glowinski1975approximation,gabay1976dual}. 

\begin{theorem}\label{thm:convergence}
Suppose that Assumption~\ref{assu} holds. Let $\{(\Gamma^k,\Omega^k,\Lambda^k,\Theta^k,\Xi^k,X^k,Y^k,U^k)\}$ be the sequence generated by Algorithm~\ref{alg: admm}. Then $\{(\Gamma^k,\Omega^k)\}$ converges to an optimal solution of the problem \eqref{eq:KSv2}.
\end{theorem}

The bottleneck in implementing Algorithm~\ref{alg: admm} is the following steps in the $k$-th iteration
\begin{align*}
	\Gamma^{k+1}
        &=\underset{\Gamma\in \mathbb{S}^{t}}{\arg\min}
	\left\{ -\frac{1}{\sigma}\log|\Omega^k \oplus \Gamma|
	+\frac{1}{2}\|\Gamma - \widetilde{\Gamma}_k \|_F^2
	\right\},\\
        \Omega^{k+1}
	  &= \underset{\Omega\in \mathbb{S}^{s}}{\arg\min}
	\left\{
	-\frac{1}{\sigma}\log|\Omega \oplus \Gamma^{k+1}|+\frac{1}{2}\|\Omega -\widetilde{\Omega}_k \|_F^2
	\right\},
\end{align*}
given some $\widetilde{\Gamma}_k\in \mathbb{S}^t$ and  $\widetilde{\Omega}_k\in \mathbb{S}^s$.
In the subsequent section, we provide an efficient procedure for this.

\subsection{Proximal Operators Associated with the Negative Log-determinant KS Function}
\label{sec: prox}

Given $\Gamma \in \mathbb{S}^t$ and $\beta>0$, we investigate the proximal operator associated with $-\beta\log|\cdot \oplus \Gamma|$ defined by
\begin{align*}
	\Psi_{{\rm Left},\beta,\Gamma}(\Omega) \!=\! \underset{\Upsilon \in \mathbb{S}^s}{\arg\min} \ \Big\{\frac{1}{2}\|\Upsilon - \Omega\|_F^2 -\beta \log|\Upsilon \oplus \Gamma|
	\Big\},
\end{align*}
for $\Omega\in \mathbb{S}^s$. The following proposition gives an efficient procedure to compute $\Psi_{{\rm Left},\beta,\Gamma}(\cdot)$.

\begin{proposition}\label{prop: left}
	Given $\beta>0$ and $\Gamma \in \mathbb{S}^t$ with eigenvalues $\lambda_1,\ldots,\lambda_t$. For any $\Omega\in \mathbb{S}^s$ with the eigenvalue decomposition $\Omega = Q\Sigma_{\Omega} Q^T$, $\Sigma_{\Omega} = {\rm Diag}(\mu_1,\ldots,\mu_s)$, we have
	\begin{align*}
		\Psi_{{\rm Left},\beta,\Gamma}(\Omega) = Q {\rm Diag}(\alpha_1,\ldots,\alpha_s) Q^T,
	\end{align*}
	where for every $j \in [s]$,  $\alpha_j$ is the unique solution to the univariate nonlinear equation
	\begin{align}
		\alpha_j-\mu_j -\sum_{i=1}^t \frac{\beta}{\alpha_j+\lambda_i}=0,
		\quad 	\alpha_j >-\min_{i \in [t]} \lambda_i.\label{eq: prox_omega}
	\end{align}
\end{proposition}
\begin{proof}
	The first part of the proof is based on the orthogonally invariant property of $\log|\cdot \oplus \Gamma|$. Namely, given $\Omega\in \mathbb{S}^s$, it holds that
	$
	\log|\Omega \oplus \Gamma| = \log| (M\Omega M^T) \oplus \Gamma|
	$
	for any orthogonal matrix $M$. This implies that for any $\Omega\in \mathbb{S}^s$ with eigenvalue decomposition $\Omega = Q\Sigma_{\Omega} Q^T$, we have
	\begin{align*}
		&\Psi_{{\rm Left},\beta,\Gamma}(\Omega)\\
&=\underset{\Upsilon \in \mathbb{S}^s}{\arg\min}\ \Big\{\frac{1}{2}\|\Upsilon - Q\Sigma_{\Omega} Q^T\|_F^2 -\beta \log|\Upsilon \oplus \Gamma|
		\Big\} \\
		&= \underset{\Upsilon \in \mathbb{S}^s}{\arg\min}\ \Big\{\frac{1}{2}\|Q^T\Upsilon Q - \Sigma_{\Omega} \|_F^2 \!-\!\beta \log|(Q^T\Upsilon Q) \oplus \Gamma| \Big\} \\
&= Q\Psi_{{\rm Left},\beta,\Gamma}(\Sigma_{\Omega})Q^T.	
	\end{align*}
	The above equality for the orthogonally invariant functions can also been found in Eqn (6.11) of \cite{parikh2014proximal}.
	Since $\Sigma_{\Omega}$ is a diagonal matrix and $\log|\cdot \oplus \Gamma|$ is orthogonally invariant, we can see that  $\Psi_{{\rm Left},\beta,\Gamma}(\Sigma_{\Omega})$ is also a diagonal matrix. Moreover, $\Psi_{{\rm Left},\beta,\Gamma}(\Sigma_{\Omega})={\rm Diag}(\alpha)$ satisfies
	$
	\textstyle \alpha = \underset{\alpha \in \mathbb{R}^s}{\arg\min} \Big\{
	\sum_{j=1}^s (\alpha_j-\mu_j)^2/2-\beta \sum_{i=1}^t \sum_{j=1}^s \log (\alpha_j+\lambda_i)\Big\},
	$
	which holds due to the fact the eigenvalues of the KS of two matrices are the pairwise sums of the eigenvalues of the two matrices \citep{horn1991topics}. Note that the above minimization problem can be solved component-wisely.  Thus $\alpha_j$ should satisfy \eqref{eq: prox_omega} due to the first-order optimality condition, for every $j \in [s]$.
	Lastly, we prove that for any given $x\in\mathbb{R}$ the equation $h(y)= (y-x)/\beta- \sum_{i=1}^{t} 1/(y+\lambda_i)=0$ admits a unique solution on the interval $(-\min_{i}\lambda_i,+\infty)$. It is true because $h$ is increasing on $(-\min_{i}\lambda_i,+\infty)$, $\lim_{y\to (-\min_{i}\lambda_i)+} h(y) = -\infty$, and $\lim_{y\to +\infty} h(y) = +\infty$.
	This completes the proof.
\end{proof}

Given $\lambda_1, \ldots, \lambda_{t}$ and $\beta>0$, let the univariate function $\psi(\cdot\ ;\lambda_1,\dots,\lambda_{t},\beta):\mathbb{R}\to \mathbb{R}$ be defined as
\begin{align}
	&\psi(x;\lambda_1,\dots,\lambda_{t},\beta)\notag \\
	&:=
	\Bigg\{y \ \Bigg| \ \frac{y-x}{\beta} = \sum_{i=1}^{t} \frac{1}{y+\lambda_i},
	y>-\min_{i \in [t]}\lambda_i\Bigg\}.\label{def-psi}
\end{align}
which is well-defined according to the proof of Proposition \ref{prop: left}. Moreover, the function value of $\psi(\cdot\ ;\lambda_1,\dots,\lambda_{t},\beta)$ can be calculated by the Newton's method or the bisection method. By solving $s$ univariate nonlinear equations, we obtain that
$
\alpha_j = \psi(\mu_j;\lambda_1,\dots,\lambda_{t},\beta),\,j \in [s].
$

Similarly, given $\Omega \in \mathbb{S}^s$ and $\beta>0$, the proximal operator associated with $-\beta \log|\Omega \oplus \cdot|$ is
 \begin{align*}
 	\Psi_{{\rm Right},\beta,\Omega}(\Gamma)
 	&\!=\! \underset{\Delta \in \mathbb{S}^t}{\arg\min} \ \Big\{\frac{1}{2}\|\Delta - \Gamma\|_F^2 -\beta \log|\Omega \oplus \Delta|
 	\Big\},
 \end{align*}
for $\Gamma \in \mathbb{S}^t$.
Analogous to Proposition \ref{prop: left}, we can provide an efficient procedure to compute $\Psi_{{\rm Right},\beta,\Omega}(\cdot)$. Details are in Appendix~\ref{sec:prox_psi}.

\begin{algorithm}[H]
	\caption{: {\tt DNNLasso}}
	\label{alg: admm}
	\begin{algorithmic}
		\STATE {\bfseries Input:} Given sample covariance matrices $R\in\mathbb{S}^t_+$, $W\in \mathbb{S}^s_+$ and a parameter $\lambda_0>0$.
		\STATE {\bfseries Initialization:} Set $\lambda_T = \lambda_0 s$, $\lambda_S = \lambda_0 t$, $\tau = 1.618$. Set $k\leftarrow 0$. Choose $\sigma > 0$ and an initial point $(\Omega^0,\Lambda^0,\Theta^0,X^0,Y^0,U^0)\in  \mathbb{S}^{s}\times\mathbb{S}^{t}\times \mathbb{S}^{s}\times \mathbb{S}^{t}\times \mathbb{S}^{s}\times \mathbb{S}^{s}$.
		
		\REPEAT
		\STATE {\bfseries Step 1}. Compute
		\begin{align*}
			&\Gamma^{k+1}=\Psi_{{\rm Right},1/\sigma,\Omega^k}(\Lambda^k +\frac{X^k}{\sigma} -\frac{R}{\sigma}),\\
			& \Xi^{k+1} = \frac{1}{2} ( \Theta^k + \Omega^k + \frac{Y^k+U^k-W}{\sigma} ).
		\end{align*}
		\\
		\STATE {\bfseries Step 2}. Let $\widetilde{\Lambda}= \Gamma^{k+1} -X^k/\sigma$, $\widetilde{\Theta}= \Xi^{k+1} - Y^k/\sigma$. Compute $\Lambda^{k+1}\in \mathbb{S}^t$ and $\Theta^{k+1},\Omega^{k+1}\in \mathbb{S}^s$ as
		\begin{align*}
			\Lambda^{k+1}_{ij} &= \left\{
			\begin{aligned}
				&\max(0, \widetilde{\Lambda}_{ii}) &&\mbox{if }i = j,\\
				&{\rm sgn}( \widetilde{\Lambda}_{ij})\max(| \widetilde{\Lambda}_{ij}|- \frac{\lambda_T}{\sigma},0) &&\mbox{if }i \neq j,
			\end{aligned}
			\right. \\
			\Theta^{k+1}_{ij} &= \left\{
			\begin{aligned}
				&\max(0, \widetilde{\Theta}_{ii}) &&\mbox{if }i = j,\\
				&{\rm sgn}( \widetilde{\Theta}_{ij})\max(| \widetilde{\Theta}_{ij}|-\frac{\lambda_S}{\sigma},0) &&\mbox{if }i \neq j,
			\end{aligned}
			\right. \\
			\Omega^{k+1} &=\Psi_{{\rm Left},1/\sigma,\Gamma^{k+1}}(\Xi^{k+1}  -\frac{U^k}{\sigma}).
		\end{align*}
		\\[5pt]
		\STATE {\bfseries Step 3}. Update the multipliers by
		\begin{align*}
			&X^{k+1} = X^k-\tau\sigma (\Gamma^{k+1} - \Lambda^{k+1}),\\
			& Y^{k+1} = Y^k-\tau\sigma (\Xi^{k+1} - \Theta^{k+1}),\\
			& U^{k+1} = U^k-\tau\sigma (\Xi^{k+1} - \Omega^{k+1}).
		\end{align*}
		\\
		\STATE {\bfseries Step 4}. Set $k\leftarrow k+1$.
		\UNTIL{Stopping criterion is satisfied.}
		
		\STATE {\bfseries Output:} 
An approximate solution $(\widehat{\Gamma}, \widehat{\Omega})$ 
computed as follows: $(\widehat{\Gamma}, \widehat{\Omega})=(\Gamma^{k},\Omega^{k})$ if $\Gamma^{k} \succ 0$, $\Omega^{k} \succ 0$; and $(\widehat{\Gamma}, \widehat{\Omega})=(\Gamma^{k}-  c I_t,\Omega^{k} + c I_s)$ with $c = (\lambda_{\rm min}(\Gamma^{k})-\lambda_{\rm min}(\Omega^{k}))/2$ otherwise.
	\end{algorithmic}
\end{algorithm}

\subsection{The Full Algorithm}
We provide the pseudocode of {\tt DNNLasso} in Algorithm \ref{alg: admm}, where the computation of $\Psi_{{\rm Right},\beta,\Omega}(\cdot)$ and $\Psi_{{\rm Left},\beta,\Gamma}(\cdot)$ is described in Section \ref{sec: prox}. Note that the non-negative constraints on diagonal elements only bring in the computation of $\max(0,\widetilde{\Lambda}_{ii})$ and $\max(0,\widetilde{\Theta}_{ii}) $, of which the cost is negligible.

We provide in Table~\ref{table: comparison} a comparison of {\tt DNNLasso} with three existing methods  {\tt BiGLasso}, {\tt TeraLasso}, and {\tt EiGLasso}, in terms of memory cost and computational cost per iteration.

\begin{table}[!ht]
	\renewcommand{\arraystretch}{1.1}
\caption{Comparison among algorithms in terms of memory cost and computational cost per iteration.} \label{table: comparison}
\begin{tabular}{ |c|c|c|  }
	\hline
	& Memory cost & Computational cost \\
	\hline
	{\tt BiGLasso}   & $O(t^2s^2)$    & $O(N_1t^3+N_1s^3)$\\
	\hline
	{\tt TeraLasso} & $O(ts+t^2+s^2)$ & $O(2ts+t^3+s^3)$ \\
	\hline
	{\tt EiGLasso} & $O(Kt^2+Ks^2)$ & $O(N_2Kt^3+N_2Ks^3)$\\
	\hline
	{\tt DNNLasso} & $O(t^2+s^2)$ & $O(t^3+s^3)$ \\
	\hline
\end{tabular}
\end{table}

In Table~\ref{table: comparison},  (1) $N_1$ represents the average number of iterations of the subroutines in {\tt BiGLasso} (i.e., the coordinate descent procedure implemented in {\tt GLasso} to estimate the precision
matrix for a simple graph); (2) $K\leq \min(t,s)$ is a user-specified parameter in the Hessian approximation of {\tt EiGLasso}, which is typically chosen within the range from $1$ to $10$. The default setting is $K=1$ in their codes. 
(3) $N_2$ represents the average number of iterations of the subroutines in {\tt EiGLasso} (i.e., the coordinate descent method to compute the Newton directions 
by minimizing second-order approximations of the objective function).

\section{NUMERICAL EXPERIMENTS}
\label{sec: experiment}
We compare our {\tt DNNLasso} with {\tt TeraLasso}\footnote{\url{https://github.com/kgreenewald/teralasso}} \citep{greenewald2019tensor} and {\tt EiGLasso}\footnote{\url{https://github.com/SeyoungKimLab/EiGLasso}} \citep{yoon2022eiglasso} on both synthetic and real data, and we use their default settings for parameters and initialization. All experiments were conducted in Matlab (version 9.11) on a Windows workstation (32-core, Intel Xeon Gold 6226R @ 2.90GHz, 128 Gigabytes of RAM).
Since the ADMM is a primal-dual method, we terminate it when the relative KKT error is less than a given tolerance,  for example, $10^{-6}$. Details are in Appendix~\ref{sec:relativekkt}.

As a warm-start of our implementation, we first run a simpler variant of {\tt DNNLasso} by eliminating the auxiliary variables $\Xi$. As we can see from \eqref{eq:KSv3}, $\Xi$ is a ``duplicate'' of the column-wise precision matrix $\Omega$, and at the optimal point they should be identical, i.e., $\Xi = \Omega$. Without $\Xi$,  this variant is simpler and has less variables compared with {\tt DNNLasso}. The pseudocode of this variant is given in Appendix~\ref{sec:admm-3block}.

\subsection{Synthetic Data}
We use two types of graph structures by \cite{yoon2022eiglasso} for generating the ground truth $\Gamma\in\mathbb{S}^{t}_{++}$ (the same for $\Omega\in\mathbb{S}^{s}_{++}$). And we sample $n$ observations from the Gaussian distribution ${\cal N}(0, (\Omega \oplus \Gamma)^{-1})$.

\noindent {\bf Type 1.} We first generate a sparse matrix $A\in\mathbb{R}^{t\times t}$ where $P(A_{ij} = -1)=\frac{1}{2}(1-\rho)$, $P(A_{ij} = 1)=\frac{1}{2}(1-\rho)$, $P(A_{ij} = 0)=\rho$, and $\rho\geq0$ is chosen such that $A$ roughly has  $10 t$ nonzero entries. Then we set $\Gamma = AA^T + 10^{-4} I_t + {\rm diag}(d_1,\dots,d_{t})$ with $d_i$ uniformly random on  $[0,0.1]$.

\noindent {\bf Type 2.} We set $\Gamma\in\mathbb{R}^{t\times t}$ to be block diagonal which contains 10 blocks and each block is generated as a graph in {\bf Type 1}. In this case we choose $\rho$ such that there are $t$ nonzero entries in each block.


\begin{figure}[!ht]
	\centering
	\subfigure[][$s=t=500$]{\includegraphics[width=0.237\textwidth]{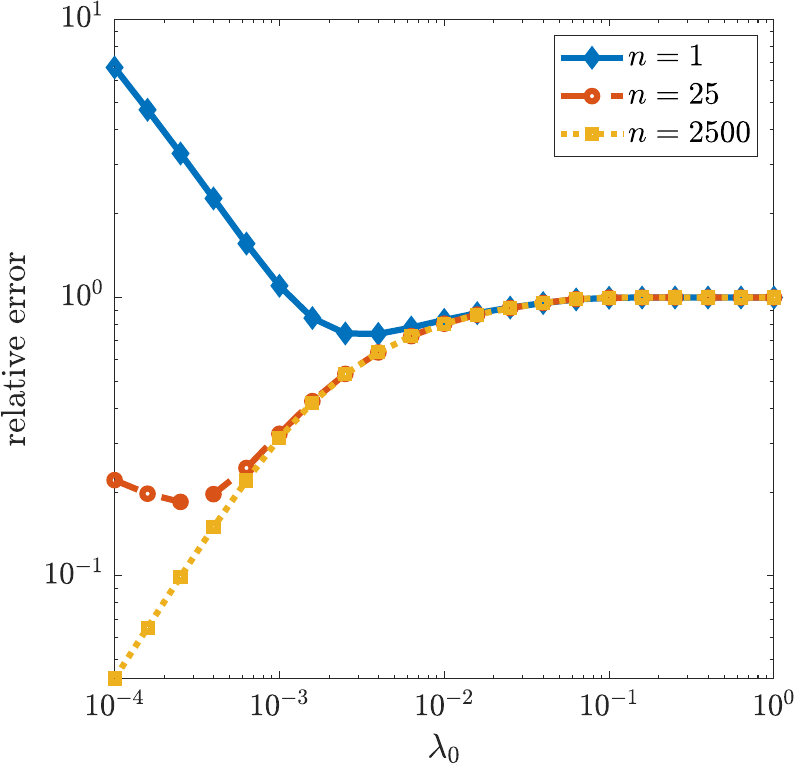}}\,
	\subfigure[][$s=t=500$]{\includegraphics[width=0.23\textwidth]{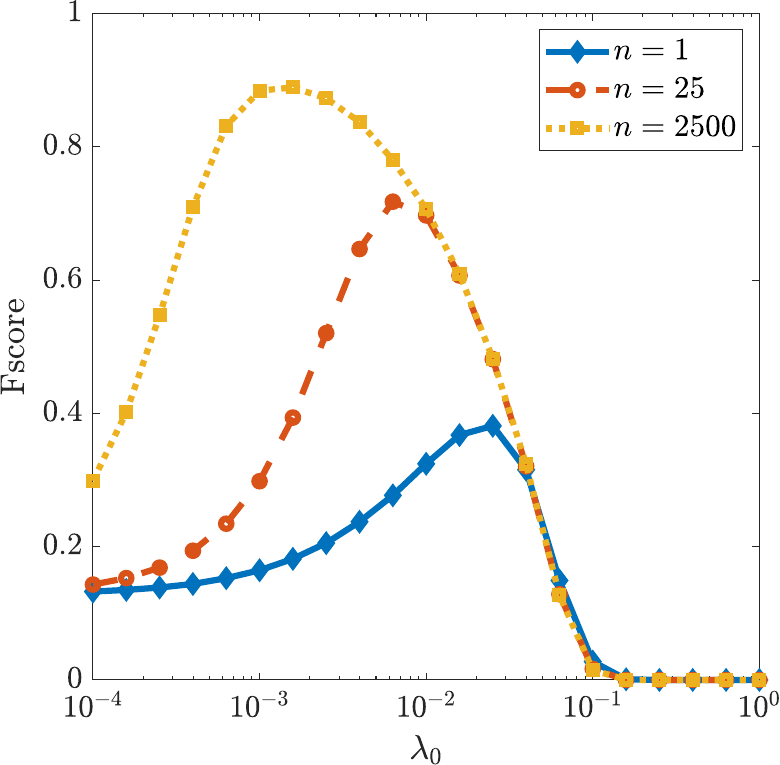}}
	\subfigure[][$s=100,t=500$]{\includegraphics[width=0.237\textwidth]{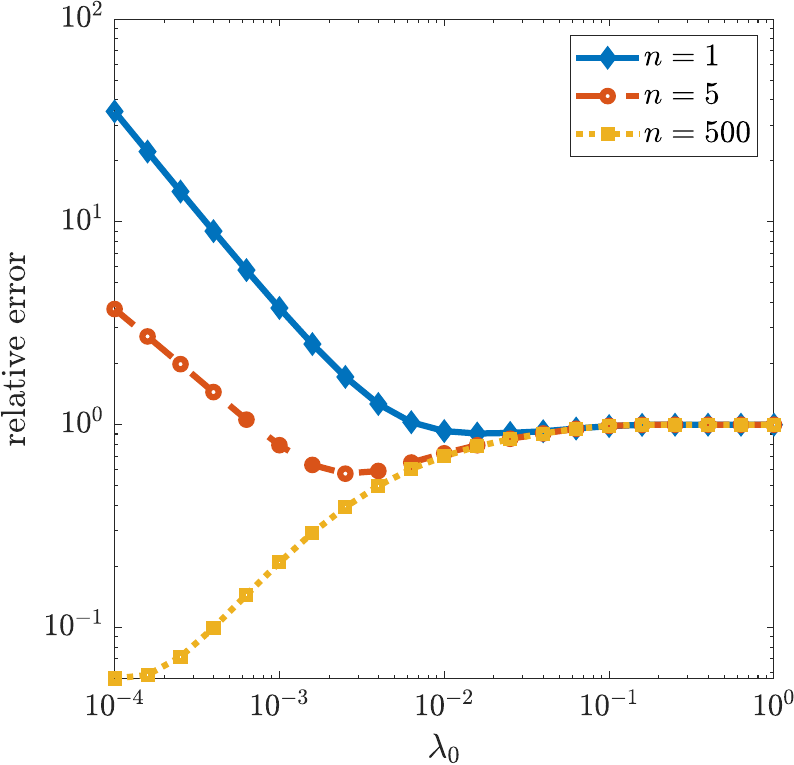}}\,
	\subfigure[][$s=100,t=500$]{\includegraphics[width=0.23\textwidth]{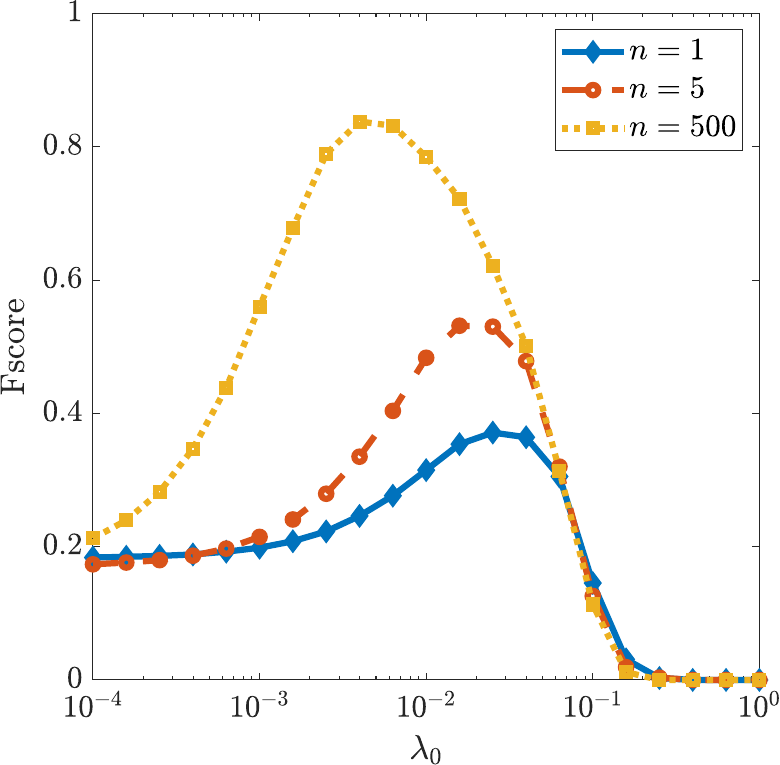}}
	\caption{Relative error (a,c) / Fscore (b,d) against $\lambda_0$ for synthetic graphs of {\bf Type 2} with dimension $s=t=500$ (a,b) / $s=100,t=500$ (c,d), and sample size $n = 1$, $st/10000$ or $st/100$.
	}\label{fig:syn-fscore}
\end{figure}

We start from medium-size graphs by considering a balanced graph size $s=t=500$ and an unbalanced one $s=100$, $t=500$. We vary the sample size $n$ in $ \{1,st/10000,st/100\}$ and select the regularization parameter $\lambda_0$ in \eqref{eq:KS} from the candidate set $\{10^{-4},10^{-3.9},10^{-3.8},\dots,10^{-0.1},10^0\}$.
We apply {\tt DNNLasso} for solving \eqref{eq:KS} with tolerance $10^{-6}$ to obtain an estimated solution $(\widetilde{\Gamma},\widetilde{\Omega})$. We compute the relative error of the estimated solution $(\widetilde{\Gamma},\widetilde{\Omega})$ with respect to the ground truth $(\Gamma,\Omega)$ via $(\|\widetilde{\Gamma}_{\rm off} - \Gamma_{\rm off}\|/\|\Gamma_{\rm off}\|+\|\widetilde{\Omega}_{\rm off} - \Omega_{\rm off}\|/\|\Omega_{\rm off} \|)/2$,
where the matrix ${\Gamma}_{\rm off}$ is constructed from $\Gamma$ by setting its diagonal entries to be zero. In addition, to measure the accuracy in identifying edges, we report the averaged Fscore, that is $({\rm Fscore}(\widetilde{\Gamma}_{\rm off}, \Gamma_{\rm off}) +{\rm Fscore}(\widetilde{\Omega}_{\rm off}, \Omega_{\rm off}) )/2$. Here Fscore$(\widetilde{\Gamma}_{\rm off}, \Gamma_{\rm off})=\frac{2{\rm tp}}{2{\rm tp}+{\rm fp}+{\rm fn}}$, where tp, fp, and fn denote the number of true positive, false positive, and false negative edges between the truth $\Gamma_{\rm off}$ and the estimator $\widetilde{\Gamma}_{\rm off}$, respectively.

Figure~\ref{fig:syn-fscore} plots the relative error and Fscore against $\lambda_0$ obtained by {\tt DNNLasso} for different dimensions and sample sizes. Overall, we can see that the relative error is smaller and the Fscore is higher for a larger sample size. In particular, when the sample size $n$ is $st/100$, Figures \ref{fig:syn-fscore}(b) and \ref{fig:syn-fscore}(d) show that the best Fscore is larger than $0.8$, which is close to the ideal value $1$.



\begin{figure}[!h]
	\centering \subfigure{\includegraphics[width=0.235\textwidth]{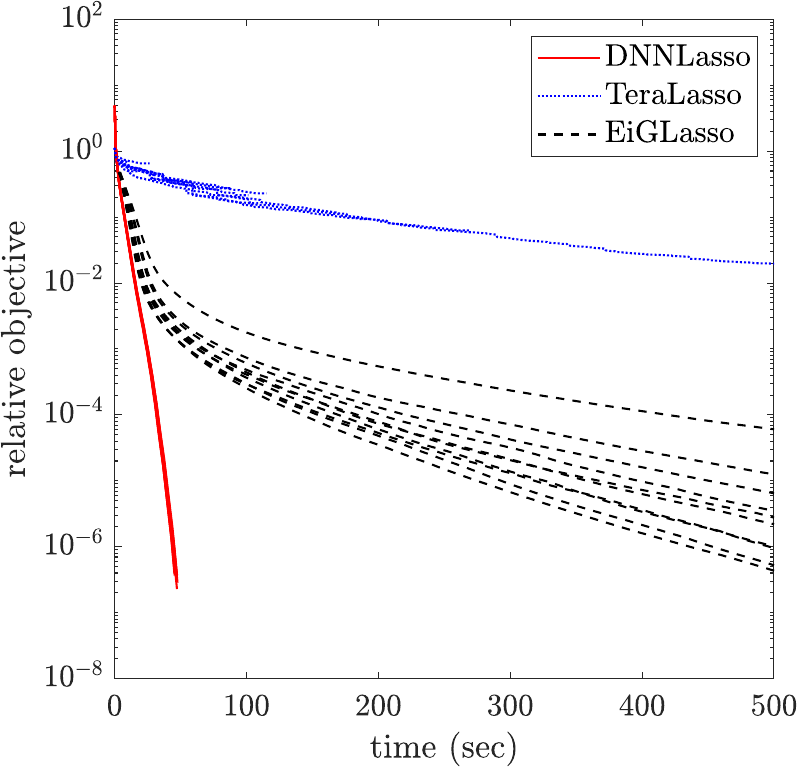}}\, \subfigure{\includegraphics[width=0.235\textwidth]{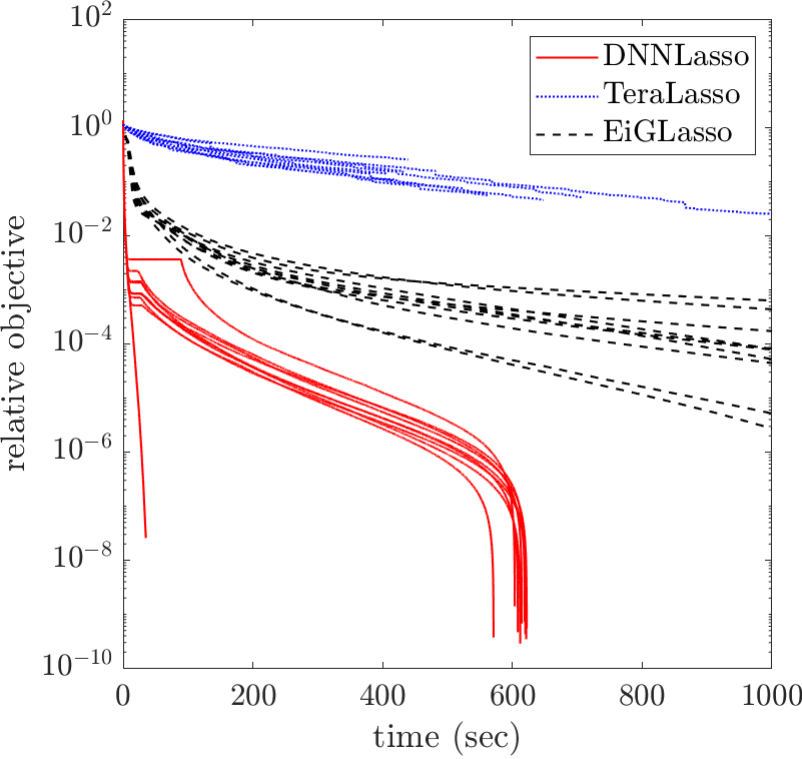}}
	\subfigure{\includegraphics[width=0.235\textwidth]{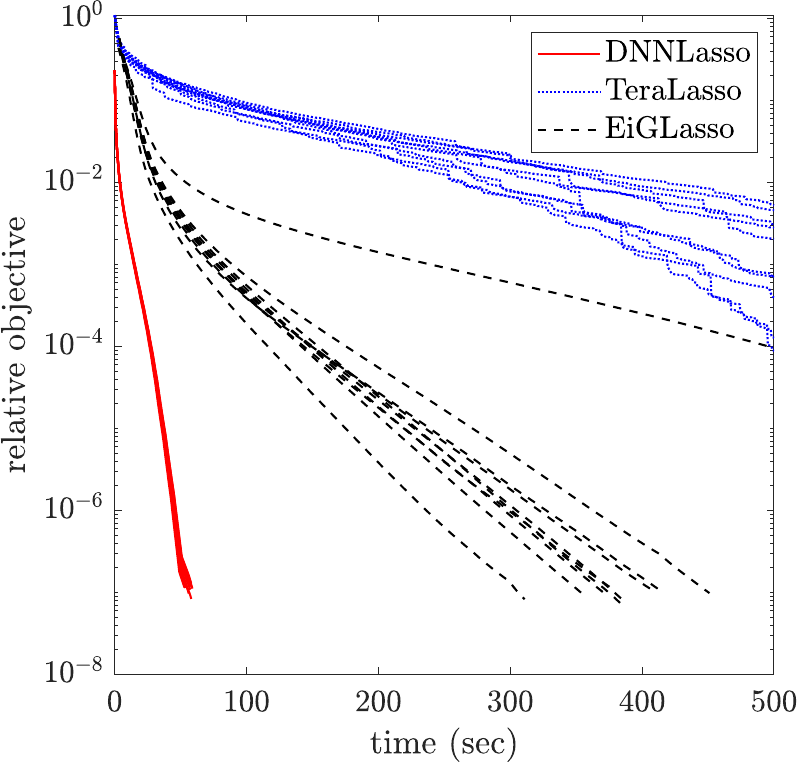}}\,
	\subfigure{\includegraphics[width=0.235\textwidth]{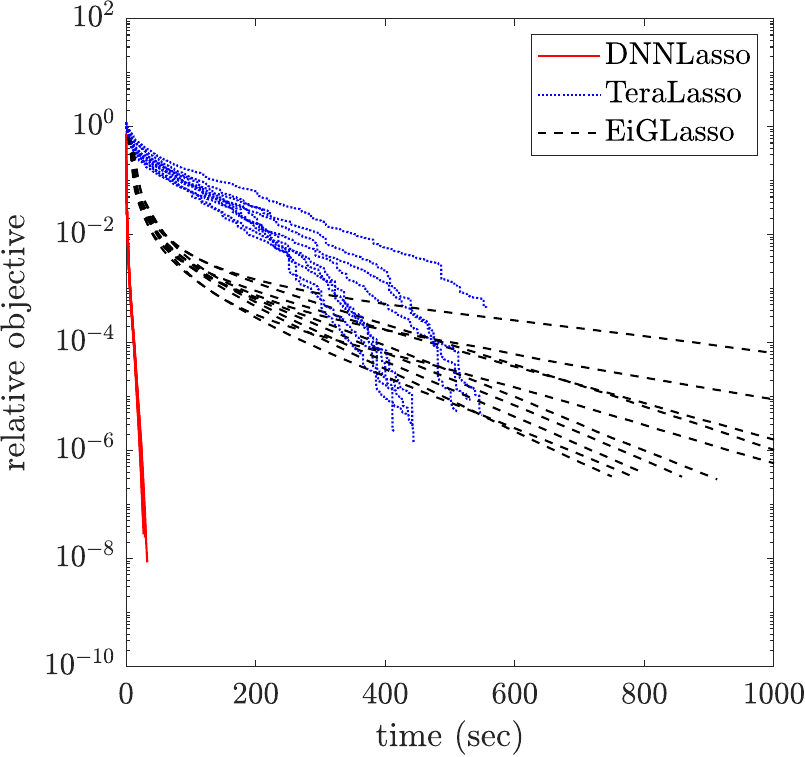}} \subfigure{\includegraphics[width=0.235\textwidth]{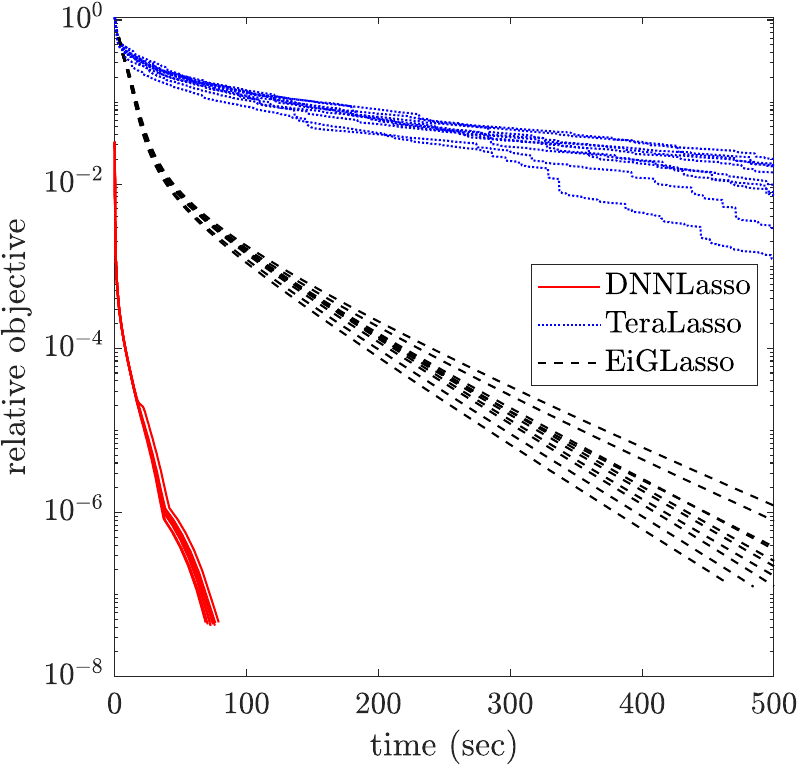}}\,
	\subfigure{\includegraphics[width=0.235\textwidth]{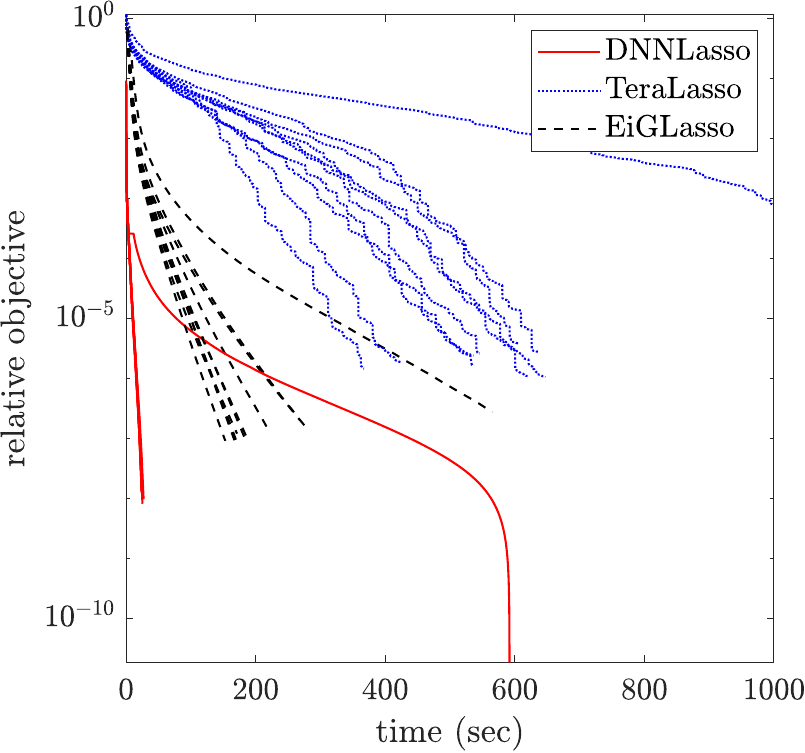}}
	\caption{Relative objective function value $(f^k - f^*)/f^*$ against time on synthetic graphs of {\bf Type 2} with dimension $s=t=500$ (left column) / $s=100,t=500$ (right column), and sample size $n = 1$, $st/10000$ or $st/100$ (rows from upper to lower). 
	}\label{fig:syn-time}
\end{figure}

We fix $\lambda_0$ to be the best parameter from the candidate set which achieves the highest Fscore (the best value can be seen from Figure~\ref{fig:syn-fscore}(b) and Figure~\ref{fig:syn-fscore}(d)). We compare the three methods {\tt DNNLasso},  {\tt TeraLasso}, and {\tt EiGLasso} on synthetic graphs of {\bf Type 2}. We use the solution of {\tt DNNLasso} with tolerance $10^{-8}$ as a benchmark and denote the corresponding objective function value as $f^*$. For the objective function value $f^k$ at the $k$-th iteration of one method, we compute the relative objective function value $(f^k - f^*)/f^*$. In Figure~\ref{fig:syn-time}, we show the relative objective function value against computational time for different methods on 10 replications. We can see from Figure~\ref{fig:syn-time} that our {\tt DNNLasso} always achieves a better objective value within a shorter time. Besides, {\tt EiGLasso} seems to be faster than {\tt TeraLasso} for a large majority of instances.

Next we compare the three methods on {\bf Type 1} graphs with relatively large balanced size $s=t=1000$ and unbalanced size $s=1000,t=400$. Since we are interested in the efficiency of each algorithm for low-sample cases, we fix the sample size $n=1$ and choose parameters $\lambda_0=10^{-2}$ or $10^{-1.6}$. Figure~\ref{fig:syn-large} illustrates the relative objective function value against computational time in different scenarios. We can see that {\tt DNNLasso} outperforms {\tt TeraLasso} and {\tt EiGLasso} by a large margin.

\begin{figure}[H]
	\centering \subfigure{\includegraphics[width=0.235\textwidth]{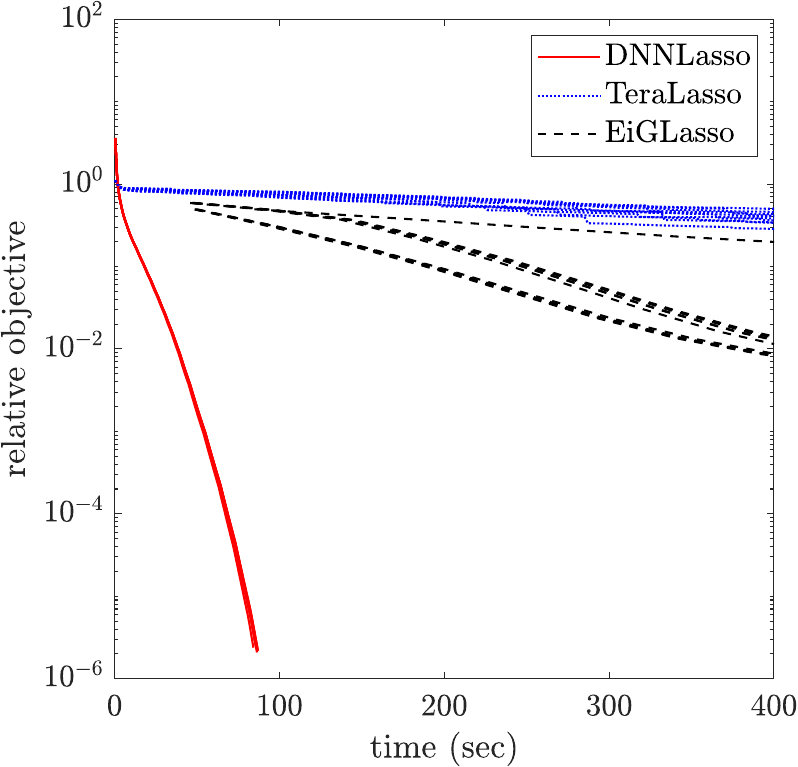}}\,
\subfigure{\includegraphics[width=0.235\textwidth]{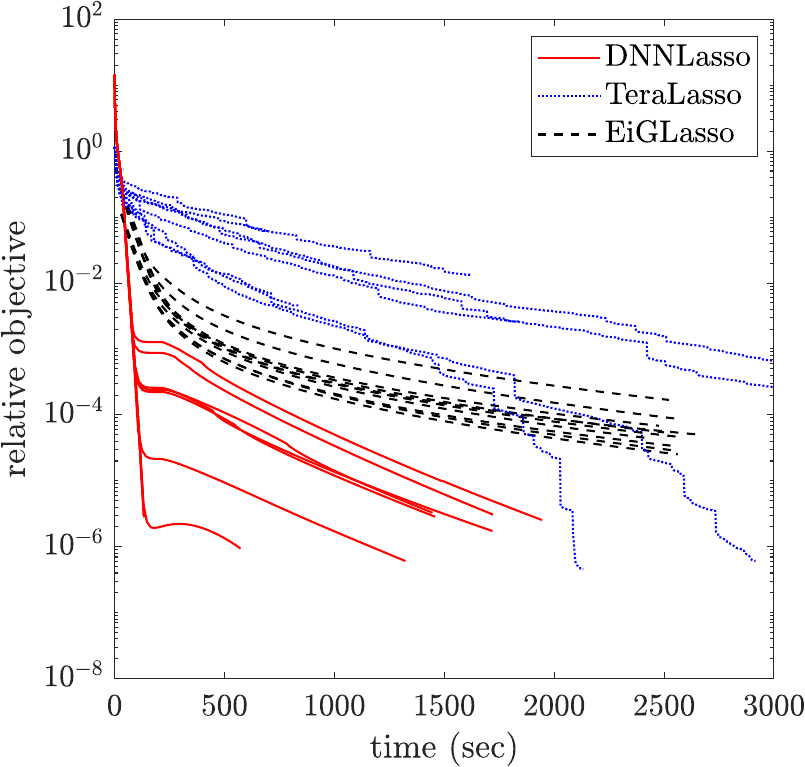}} \subfigure{\includegraphics[width=0.235\textwidth]{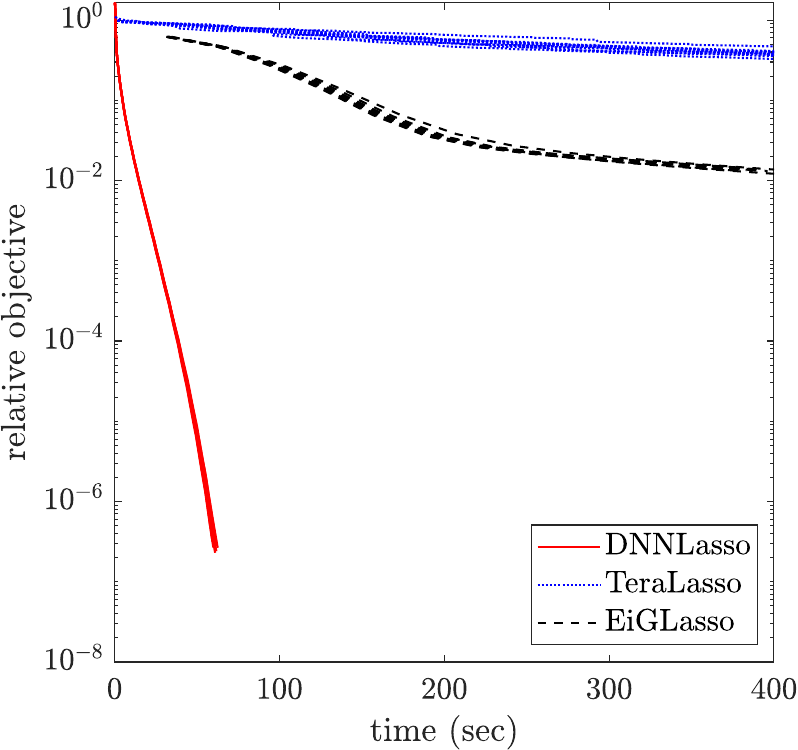}}\,
\subfigure{\includegraphics[width=0.235\textwidth]{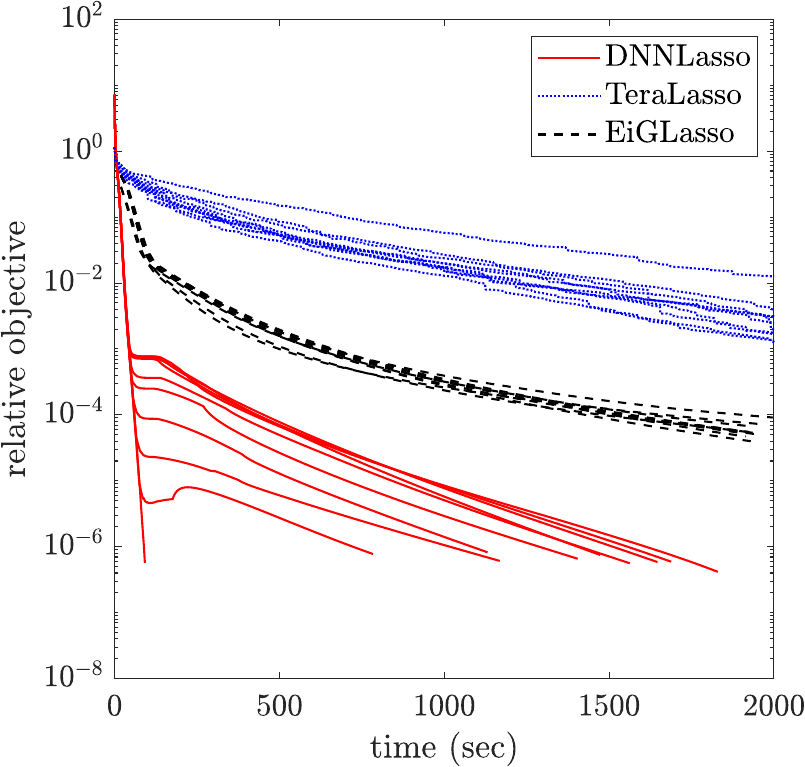}}
	\caption{Relative objective function values against time on  graphs of {\bf Type 1} with sample size $n=1$ and dimensions $s=t=1000$ (upper row) / $s=1000,t=400$ (bottom row), $\lambda_0=10^{-2}$ (left column) / $\lambda_0=10^{-1.6}$ (right column).
	}\label{fig:syn-large}
\end{figure}

\subsection{COIL100 Video Data}\label{sec-video}
In this section, we adopt the data from the Columbia Object Image Library (COIL) \citep{nene1996columbia}.
The data contains 100 objects, and for each object, it contains $s=72$ frames (color images with the resolution of $t=128\times 128$ pixels) of the rotating object from different angles (every $5^{\rm o}$). Our goal is to jointly recover the conditional dependency structure over the frames and the structure over the pixels. In consideration of the computational complexity, we choose to reduce the resolution of each frame. Likewise,  \cite{kalaitzis2013bigraphical} consider the reduced resolution of $8\times 8$. We pick one object (a box of cold medicine) from the data, which is illustrated in Figure~\ref{fig:box}. From Figure~\ref{fig:box}, we can roughly recognize the object from the compressed images of $32\times 32$ pixels in the second row, but it is hard to recognize the object from the compressed images of $8\times 8$ pixels in the third row. This implies that the reduced resolution of $32\times 32$ might be a better choice for graph learning than the reduced resolution of $8\times 8$.

\begin{figure}[!ht]
	\centering
	\includegraphics[width=0.28\textwidth]{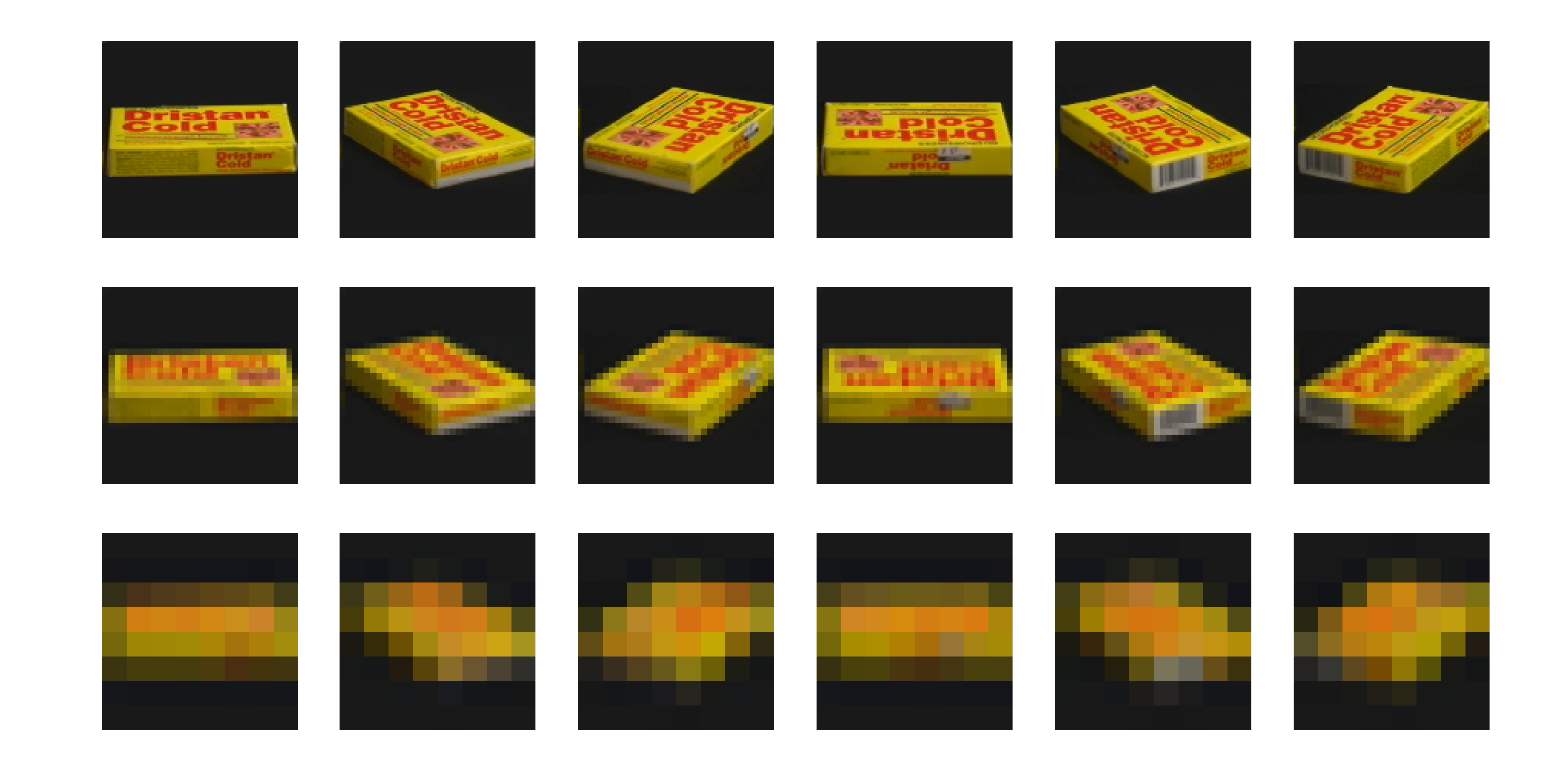}
	\caption{A rotating box of cold medicine in COIL100 video data. First row: original resolution of $128\times 128$ pixels. Second (resp. third) row: reduced resolution of $32\times 32$ (resp. $8\times 8$) pixels.}
	\label{fig:box}
\end{figure}

\begin{figure}[!ht]
	\centering
    \subfigure[][]{\includegraphics[width=0.24\textwidth]{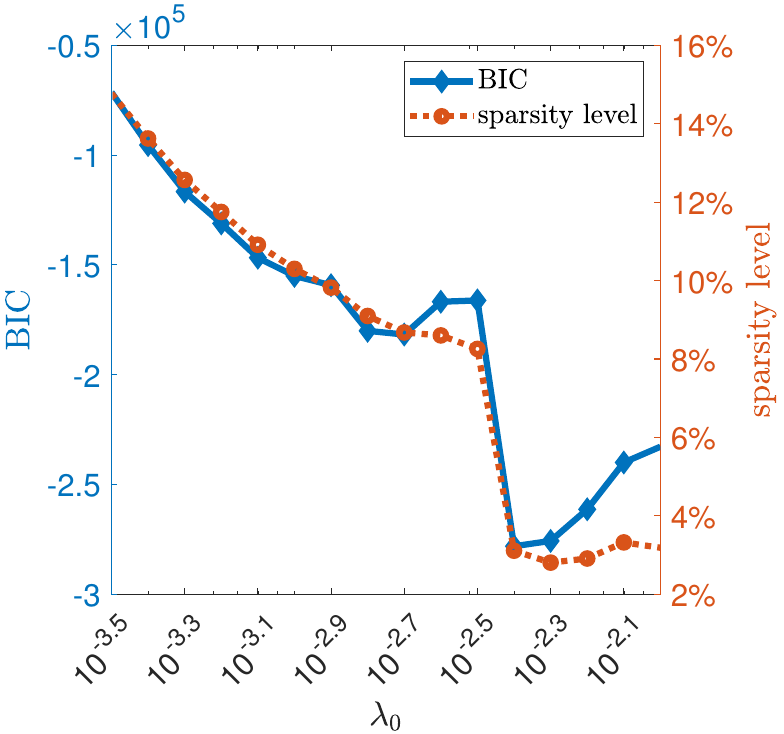}}\,\,
	\subfigure[][]{\includegraphics[width=0.22\textwidth]{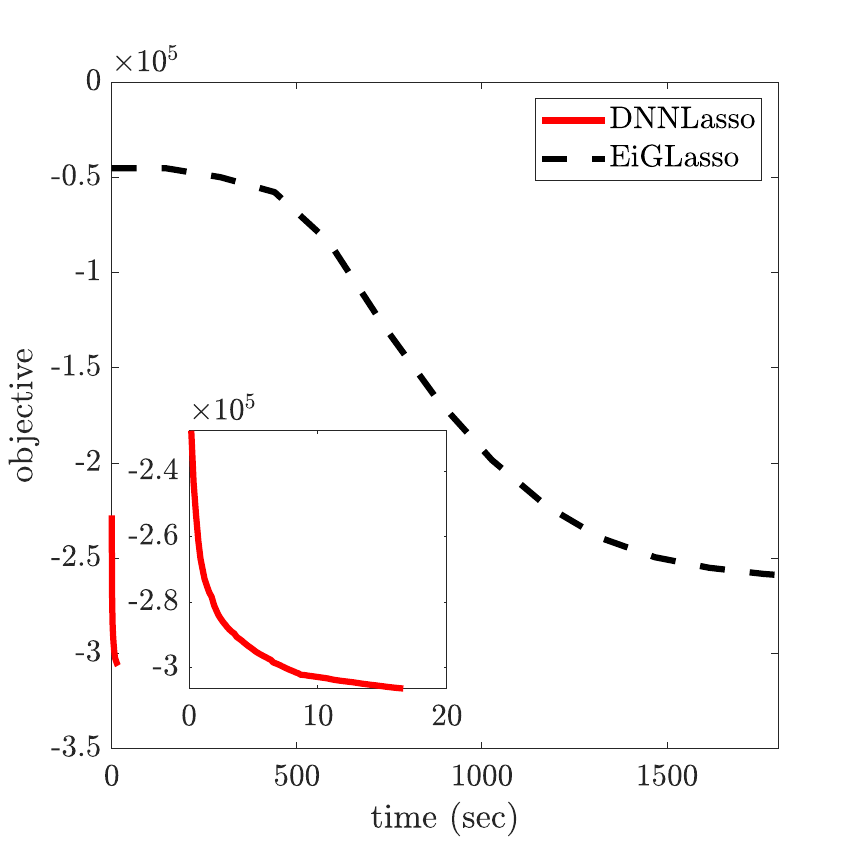}}
\\
	\subfigure[][]{
\includegraphics[width=0.19\textwidth]{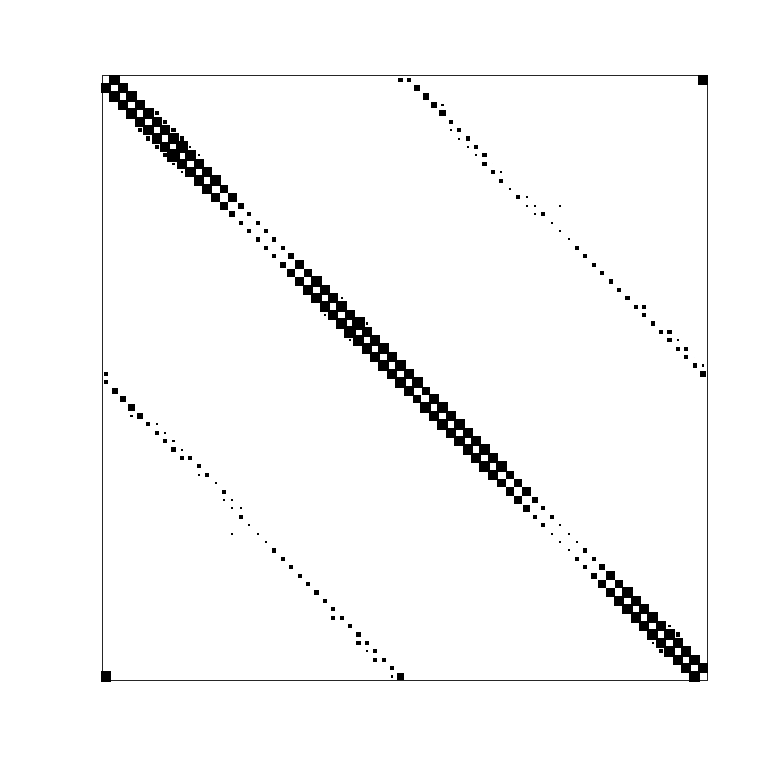}}\quad\quad
	\subfigure[][]{\includegraphics[width=0.22\textwidth]{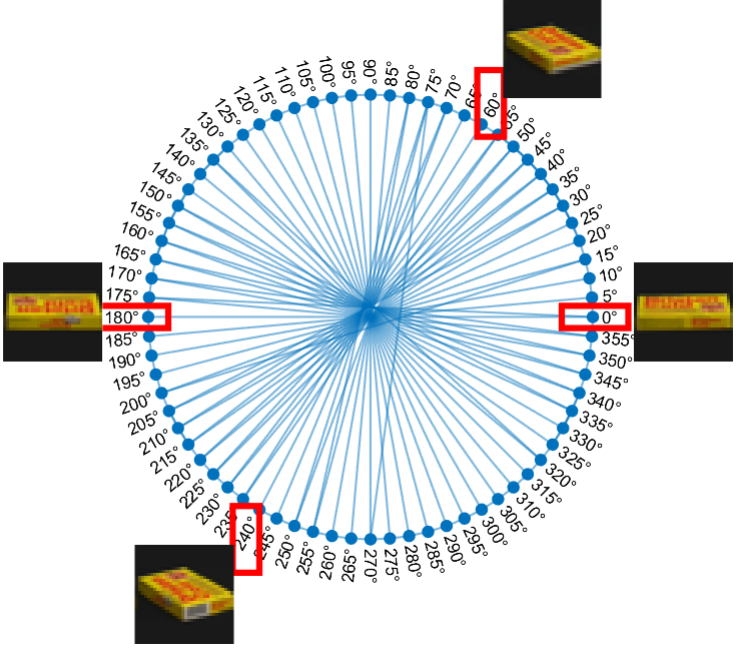}}
	\caption{On $s=72$ frames  with  $t=32\times 32$ pixels. (a) The BIC and sparsity level against $\lambda_0$. (b) The relative objective function value against computational time. 
		(c) Sparsity pattern of the  matrix $\widetilde{\Omega}\in\mathbb{S}^s$ estimated by {\tt DNNLasso} (i.e., the correlation pattern among frames from different angles).  (d) Relationship graph of frames from different angles.}
	\label{fig:coil1}
\end{figure}

\begin{figure}[!ht]
	\centering
\subfigure[][]{\includegraphics[width=0.24\textwidth]{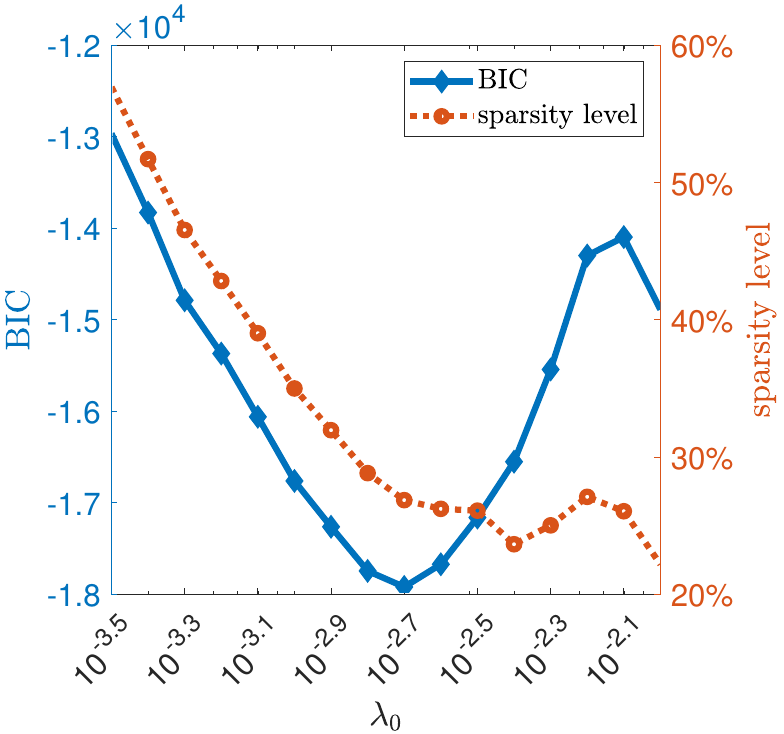}}\,\,
\subfigure[][]{\includegraphics[width=0.22\textwidth]{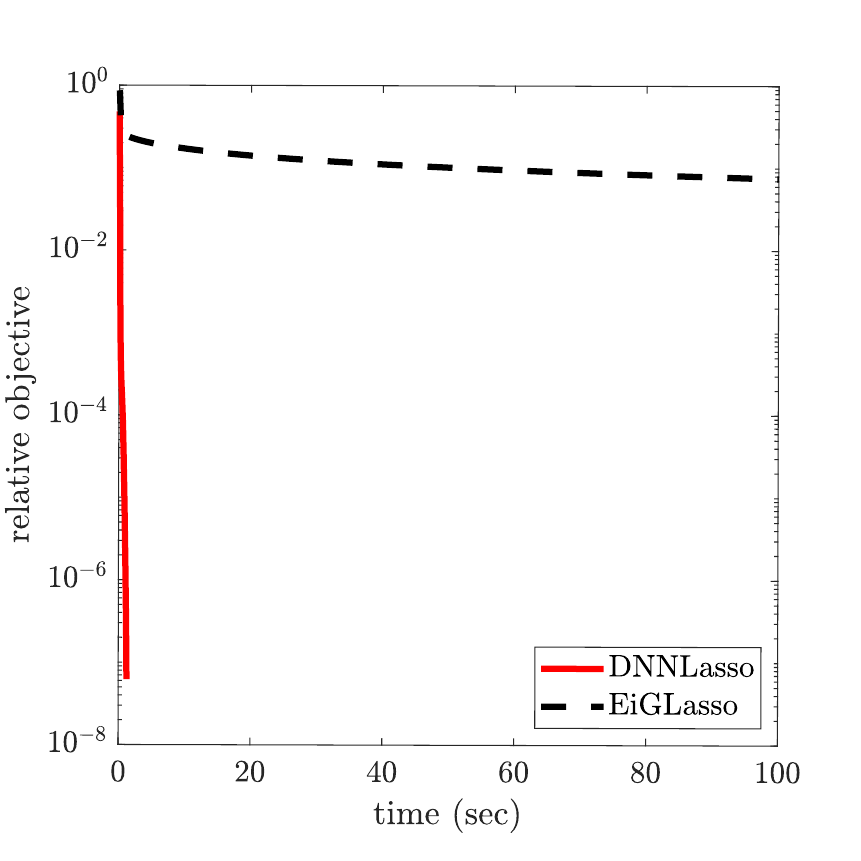}}
\\
\subfigure[][]{\includegraphics[width=0.19\textwidth]{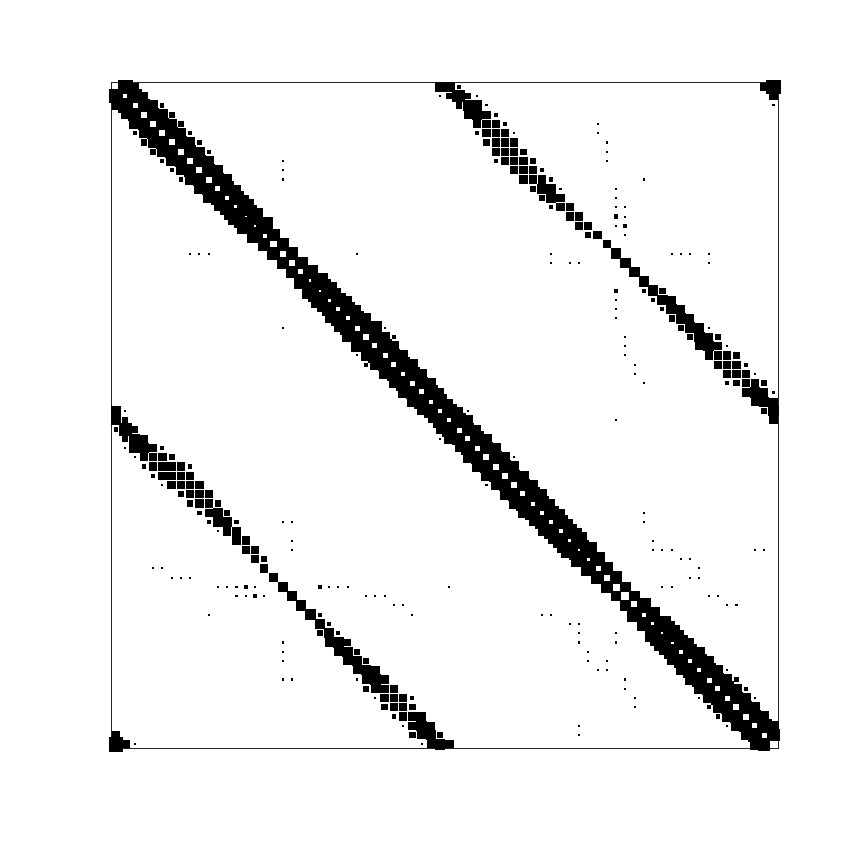}}\quad\quad
\subfigure[][]{\includegraphics[width=0.22\textwidth]{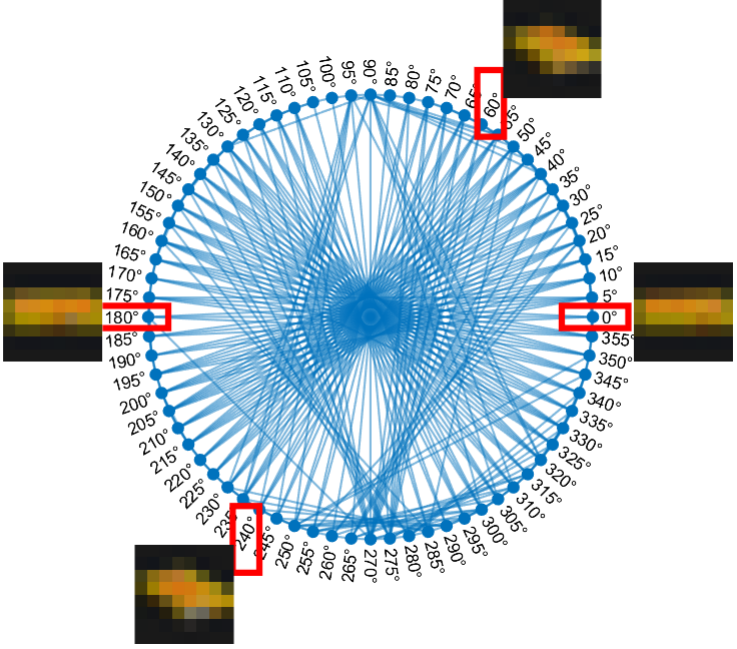}}
	\caption{On $s=72$ frames  with $t=8\times 8$ pixels. }
	\label{fig:coil1_small}
\end{figure}

We first conduct experiments on $s=72$ frames with the reduced resolution of $t=32\times 32$ pixels. We select parameter $\lambda_0$ from the set $\{10^{-3.5},10^{-3.4},\dots,10^{-2.1},10^{-2}\}$ under the Bayesian information criterion (BIC), and then compare our {\tt DNNLasso} with {\tt TeraLasso} and {\tt EiGLasso}. We terminate {\tt DNNLasso} with tolerance $5\times 10^{-3}$.

Figure~\ref{fig:coil1}(a) plots the BIC value and sparsity level against $\lambda_0$. For an estimated pair $(\widetilde{\Gamma},\widetilde{\Omega})$, the  BIC value is computed as
$
{\rm BIC}(\widetilde{\Gamma},\widetilde{\Omega}) = \!-\!\log|\widetilde{\Omega} \!\oplus\! \widetilde{\Gamma}| \!+\! \langle \widetilde{\Omega}, W\rangle \!+\! \langle \widetilde{\Gamma}, R\rangle  + (0.5\log(n)/n + 0.2 \log(st)) (  \|\widetilde{\Omega}\|_{0,{\rm off}} + \|\widetilde{\Gamma}\|_{0,{\rm off}}  ),
$
and the sparsity level is computed as
$( \|\widetilde{\Omega}\|_{0,{\rm off}} + \|\widetilde{\Gamma}\|_{0,{\rm off}})/(s(s-1) + t(t-1)), $
where $\|\cdot\|_{0,{\rm off}}$ denotes the the number of nonzero off-diagonal entries in a matrix. We can see from Figure~\ref{fig:coil1}(a) that the sparsity level is roughly decreased from $15\%$ to $3\%$ as $\lambda_0$ increases and $\lambda_0=10^{-2.4}$ achieves the best BIC. Figure~\ref{fig:coil1}(b) illustrates the objective function value against computational time  with $\lambda_0=10^{-2.4}$. We do not include {\tt TeraLasso} in this instance since it failed to return a positive definite solution $\Omega\oplus \Gamma$. From Figure~\ref{fig:coil1}(b) we can see that {\tt DNNLasso} took less than 20 seconds and achieved a much better objective function value than {\tt EiGLasso} after more than 1500 seconds. In Figure~\ref{fig:coil1}(c,d), we demonstrate the sparsity pattern of the  matrix $\widetilde{\Omega}\in\mathbb{S}^s$ estimated by {\tt DNNLasso}, namely the relationship graph of frames from different angles.
Here, we only show the off diagonal entries of $\widetilde{\Omega}$, whereby zero values ($\widetilde{\Omega}_{ij}=0$) are white, nonzero values ($\widetilde{\Omega}_{ij}\neq0$) are represented by black squares, and the size of a black square is proportional to the weight $|\widetilde{\Omega}_{ij}|$. The relationship graph in Figure~\ref{fig:coil1}(c) indicates that an image observed from $x^{\rm o}$ is connected not only to adjacent images from $(x \pm 5)^{\rm o}$, but also to images from $(x\pm180)^{\rm o}$. This observation is intuitively right as the object is a box. The two images from $0^{\rm o}$ and $180^{\rm o}$ (and also those from $60^{\rm o}$ and $240^{\rm o}$) have similar exteriors, as plotted in Figure~\ref{fig:coil1}(d).

Besides, we report the experimental results on the reduced resolution data with $t=8\times 8$ pixels. 
We find that there are some unexpected correlations among frames shown in Figure~\ref{fig:coil1_small}(d), compared with Figure~\ref{fig:coil1}(d). One possible reason is that images of $8\times 8$ pixels are too blur to identify.

More results on the synthetic data and video data are in Appendix~\ref{sec:largedata} and Appendix~\ref{sec:coil100}, respectively.

\section{CONCLUSION}
In this paper, we propose {\tt DNNLasso}, an efficient framework for estimating the KS-structured precision matrix for matrix-variate data. We develop an efficient and robust ADMM based algorithm for solving it 
and derive an explicit solution of proximal operators associated with the negative log-determinant of KS function for the first time. Numerical experiments demonstrate that {\tt DNNLasso} is superior to the existing methods by a large margin. However, our algorithm still relies on the eigenvalue decompositions in each iteration. In future work, we will consider partial or certain economical eigenvalue decomposition to further reduce computational cost. Additionally, we acknowledge that our algorithm is a first-order method that may not be efficient enough for obtaining highly accurate solutions. Therefore, we may include some second-order information into the algorithm to further speed up the process.

\subsection*{Acknowledgements}
Meixia Lin is supported by the Ministry of Education, Singapore, under its Academic Research Fund Tier 2 grant call (MOE-T2EP20123-0013) and the Singapore University of Technology and Design under MOE Tier 1 Grant SKI 2021\_02\_08.
Yangjing Zhang is supported by the National Natural Science Foundation of China under grant
number 12201617.

\bibliography{references}
\bibliographystyle{abbrvnat}

\appendix

\onecolumn
\aistatstitle{Supplementary Materials}
\aistatsauthor{ }

\section{PROOF OF PROPOSITION \ref{prop: equivalence}}
\label{sec:prof_of_equivalence}
\begin{proof}
We first give some notations. Denote the function
\begin{align*}
f_0(\Gamma,\Omega) &= -\log|\Omega \oplus \Gamma| + \langle \Omega, W\rangle + \langle \Gamma, R\rangle  +\lambda_0 s \|\Gamma\|_{1,{\rm off}} +\lambda_0 t \|\Omega\|_{1,{\rm off}} \\
&=  -\log|\Omega \oplus \Gamma| + \langle \Omega \oplus \Gamma,C \rangle + \lambda_0 \|\Omega \oplus \Gamma\|_{1,{\rm off}},
\end{align*}
where $C$ is the sample covariance matrix in \eqref{eq: graphical_lasso}. Denote the optimal objective function values of \eqref{eq:KS} and \eqref{eq:KS_NN} as $f_2^*$ and $f_3^*$, respectively. Denote the feasible set of \eqref{eq:KS} as
\begin{align*}
\mathcal{F}_2 = \{(\Gamma,\Omega)\mid \Gamma \in\mathbb{S}^t_{++},\ \Omega \in\mathbb{S}^s_{++}\},
\end{align*}
and the feasible set of \eqref{eq:KS_NN} as
\begin{align*}
\mathcal{F}_3 = \{(\Gamma,\Omega)\mid \Omega \oplus \Gamma \in\mathbb{S}^{st}_{++},\ {\rm diag}(\Omega) \geq 0,\ {\rm diag}(\Gamma) \geq 0\}.
\end{align*}

(a) Obviously, we can see from the definition that $\mathcal{F}_2 \subseteq \mathcal{F}_3$, and thus
\begin{align*}
	f_2^*\geq f_3^*.
\end{align*}
Suppose $(\Gamma^*,\Omega^*)$ is an optimal solution to \eqref{eq:KS_NN} and then $f_3^* = f(\Gamma^*,\Omega^*)$. By our construction of $(\widehat{\Gamma},\widehat{\Omega})$, it follows from the non-identifiability of diagonals that
\begin{align*}
	\widehat{\Omega} \oplus \widehat{\Gamma} = \Omega^*\oplus \Gamma^*
\end{align*}
and therefore $f(\widehat{\Gamma},\widehat{\Omega}) = f(\Gamma^*,\Omega^*)$. Moreover,
\begin{align*}
	\lambda_{\min}(\Omega^*\oplus \Gamma^*) = \lambda_{\min}(\Omega^*) + \lambda_{\min}(\Gamma^*) >0,
\end{align*}
and by the choice of $c$, we have $\widehat{\Gamma} \in \mathbb{S}^{t}_{++},\widehat{\Omega} \in \mathbb{S}^{s}_{++}$. Namely, $(\widehat{\Gamma},\widehat{\Omega})\in \mathcal{F}_2$ is a feasible point to \eqref{eq:KS}. Then
\begin{align*}
	f_2^*\leq f(\widehat{\Gamma},\widehat{\Omega}) = f(\Gamma^*,\Omega^*)=f_3^*.
\end{align*}
We have proved that $f_2^*=f_3^*$.

(b) The statement holds naturally since $\mathcal{F}_2 \subseteq \mathcal{F}_3$ and $f_2^*=f_3^*$.

(c) As shown in (a), $(\widehat{\Gamma},\widehat{\Omega})$ is optimal to \eqref{eq:KS} as it is feasible and attains the optimal objective function value, that is,
\begin{align*}
	(\widehat{\Gamma},\widehat{\Omega})\in \mathcal{F}_2,\quad f(\widehat{\Gamma},\widehat{\Omega}) = f(\Gamma^*,\Omega^*)=f_3^*=f_2^*.
\end{align*}

The proof is completed.
\end{proof}

\section{PROOF OF THEOREM \ref{thm: sol}}
\label{sec:prof_of_sol}
\begin{proof}
First of all, we can prove that the function $(\Gamma, \Omega) \to -\log |\Omega \oplus \Gamma|$ is lower semi-continuous. In fact, it follows from the fact that $\alpha \geq -\log |\Omega \oplus \Gamma|$ whenever $\alpha = \lim \alpha_k$, $\Omega = \lim \Omega_k$, $\Gamma = \lim \Gamma_k$ for sequences $\{\alpha_k\}$, $\{\Omega_k\}$, $\{\Gamma_k\}$ such that $\alpha_k \geq -\log |\Omega_k \oplus \Gamma_k|$ for every $k$. Denote the objective function in \eqref{eq:KS_NN} as
\begin{align*}
f(\Gamma,\Omega)   =   \left\{\begin{aligned}
&  -  \log|\Omega  \oplus  \Gamma|   +   \langle \Omega, W\rangle   +   \langle \Gamma, R\rangle    +  \lambda_0 s \|\Gamma\|_{1,{\rm off}}   +  \lambda_0 t \|\Omega\|_{1,{\rm off}}  && \mbox{if }  {\rm diag}(\Omega) \geq  0 ,  {\rm diag}(\Gamma) \geq  0\\
&+\infty && \mbox{otherwise}
\end{aligned}\right. .
\end{align*}
Then it can be seen that $f(\cdot,\cdot)$ is lower semi-continuous, convex, proper in $\mathbb{S}^{t}\times\mathbb{S}^{s}$.

Next we compute the recession function of $f(\cdot,\cdot)$ based on \citet[Theorem~8.5]{rockafellar1996convex}. We have that
\begin{align*}
&(f0+)(\Gamma,\Omega) = \lim_{\alpha \to+\infty} \frac{f(I_{t} + \alpha\Gamma,I_{s} + \alpha \Omega) - f(I_{t},I_{s})}{\alpha}\\
&=
\begin{cases}
\langle \Omega, W\rangle + \langle \Gamma, R\rangle  +\lambda_0 s \|\Gamma\|_{1,{\rm off}} +\lambda_0 t \|\Omega\|_{1,{\rm off}} & \mbox{ if }  {\rm diag}(\Gamma)  \geq  0,\ {\rm diag}(\Omega)  \geq  0,\ \Omega \oplus \Gamma \in\mathbb{S}^{st}_{+}\\
+\infty & \mbox{ otherwise}
\end{cases},
\end{align*}
where the last equality holds since when ${\rm diag}(\Gamma) \geq 0,\ {\rm diag}(\Omega) \geq 0,\ \Omega \oplus \Gamma \in\mathbb{S}^{st}_{+}$, we have
\begin{align*}
&\lim_{\alpha \to+\infty} \frac{f(I_{t} + \alpha\Gamma,I_{s} + \alpha \Omega) - f(I_{t},I_{s})}{\alpha} \\
&=
\lim_{\alpha\to+\infty}\frac{-\log |I_{s} \oplus I_{t} + \alpha (\Omega \oplus \Gamma)| + \log|I_{s} \oplus I_{t}|}{\alpha} + \langle \Omega, W\rangle + \langle \Gamma, R\rangle  +\lambda_0 s \|\Gamma\|_{1,{\rm off}} +\lambda_0 t \|\Omega\|_{1,{\rm off}} \\
&=\langle \Omega, W\rangle + \langle \Gamma, R\rangle  +\lambda_0 s \|\Gamma\|_{1,{\rm off}} +\lambda_0 t \|\Omega\|_{1,{\rm off}}.
\end{align*}
Therefore, the recession cone of $f(\cdot,\cdot)$ is
\begin{align*}
&\{(\Gamma,\Omega) \mid {\rm diag}(\Gamma)  \geq  0, {\rm diag}(\Omega)  \geq  0, \Omega  \oplus  \Gamma \in \mathbb{S}^{st}_{+},  \langle \Omega, W\rangle  +  \langle \Gamma, R\rangle   + \lambda_0 s \|\Gamma\|_{1,{\rm off}} + \lambda_0 t \|\Omega\|_{1,{\rm off}} \leq 0\}\\
&=\{(\Gamma,\Omega) \mid {\rm diag}(\Gamma)  \geq  0,\ {\rm diag}(\Omega)  \geq  0,\ \Omega  \oplus  \Gamma \in \mathbb{S}^{st}_{+}, \ \langle \Omega \oplus \Gamma, C\rangle =0,   \|\Gamma\|_{1,{\rm off}}=0,  \|\Omega\|_{1,{\rm off}}=0 \}\\
&=\{(\Gamma,\Omega) \mid \Omega  =  {\rm  Diag}(\alpha_1,\dots,\alpha_s),\Gamma  =  {\rm Diag}(\gamma_1,\dots,\gamma_t),\alpha_i \geq  0,\gamma_j \geq  0,(\alpha_i + \gamma_j) \sum_{k=1}^n( Z_{ji}^{(k)})^2 = 0\},
\end{align*}
where the first equality follows from that $\langle \Omega, W\rangle + \langle \Gamma, R\rangle = \langle \Omega \oplus \Gamma, C\rangle \geq 0$ as both $\Omega \oplus \Gamma$ and the sample covariance matrix $C$ in \eqref{eq: graphical_lasso} are positive semidefinite; the second equality uses the expression of diagonal entries of $C$; ${\rm Diag}(\alpha_1,\dots,\alpha_s)$  returns a square diagonal matrix with the elements $\alpha_i$ on the main diagonal.

We prove by contradiction to see that $\alpha_i + \gamma_j = 0$, for all $i,j$. Suppose there exist $i_1$ and $j_1$ such that $\alpha_{j_1} + \gamma_{i_1} > 0$, then $\sum_{k=1}^n( Z_{i_1j_1}^{(k)})^2=0$. Under Assumption \ref{assu}, we have $R_{i_1i_1}>0$ and thus $ Z_{i_1\cdot}^{(k)} \neq 0$ for some $k$. Namely, there exits $j_2\neq j_1$ such that $ Z_{i_1j_2}^{(k)}\neq 0$, which implies $\sum_{k=1}^n (Z_{i_1j_2}^{(k)})^2\neq 0$ and then $\alpha_{j_2} + \gamma_{i_1}=0$. Similarly, under Assumption \ref{assu}, we have $W_{j_1j_1}>0$, which implies that $\alpha_{j_1} + \gamma_{i_2}=0$ for some $i_2\neq i_1$. Therefore,
$\alpha_{j_1} + \gamma_{i_1} = \alpha_{j_1} - \alpha_{j_2} > 0$ and  $\alpha_{j_2} + \gamma_{i_2} = \alpha_{j_2} - \alpha_{j_1} < 0$, which is contradictory to $\alpha_i + \gamma_j \geq 0$. Therefore, all $\alpha_i$'s and $\gamma_j$'s are zero and the recession cone contains zero alone.

Lastly, by   \citet[Theorem 27.1]{rockafellar1996convex}, the minimum set of $f$ is a non-empty bounded set. To this end, we have proven that problem \eqref{eq:KS_NN} admits a non-empty and bounded solution set.
\end{proof}

\section{PROXIMAL OPERATOR ASSOCIATED WITH $-\beta \log|\Omega \oplus \cdot|$}
\label{sec:prox_psi}
The following proposition provides an efficient procedure to compute $\Psi_{{\rm Right},\beta,\Omega}(\cdot)$. The proof is omitted as it is similar to the case in Proposition \ref{prop: left}.
 \begin{proposition}\label{prop: right}
 	Given $\beta>0$ and $\Omega \in \mathbb{S}^s$ with eigenvalues $\mu_1,\ldots,\mu_s$. For any $\Gamma\in \mathbb{S}^t$ with the eigenvalue decomposition $\Gamma = P\Sigma_{\Gamma} P^T$, $\Sigma_{\Gamma} = {\rm Diag}(\lambda_1,\ldots,\lambda_t)$, we have
 	\begin{align*}
 	\Psi_{{\rm Right},\beta,\Omega}(\Gamma) = P{\rm Diag}(\alpha_1,\ldots,\alpha_t)P^T,
 	\end{align*}
 	where for every $i=1,\ldots,t$,  $\alpha_i$ is the unique solution to the univariate nonlinear equation	
 	\begin{align*}
 	\alpha_i-\lambda_i -\sum_{j=1}^s \frac{\beta}{\alpha_i+\mu_j} = 0,\quad \alpha_i >- \min_{j=1,\ldots,s} \mu_j.
 	\end{align*}
 \end{proposition}

\section{RELATIVE KKT ERROR}
\label{sec:relativekkt}
Here is a remark on the stopping criterion of {\tt DNNLasso}. Note that the Karush-Kuhn-Tucker (KKT) optimality conditions of \eqref{eq:KSv3} are given as follows:
\begin{align*}
	\left\{ \begin{aligned}
   & -\nabla \Gamma + R - X=0 \quad \mbox{ where } \nabla \Gamma = \frac{\mathrm{d} \log |\Omega \oplus \Gamma|}{\mathrm{d} \Gamma}\\
   & -\nabla \Omega + U=0 \quad \mbox{ where } \nabla \Omega = \frac{\mathrm{d} \log |\Omega \oplus \Gamma|}{\mathrm{d} \Omega}\\
   & \Lambda - {\rm Prox}_p(\Lambda-X) = 0\\
   & \Theta - {\rm Prox}_q(\Theta -Y) = 0 \\
   & W - Y - U = 0\\
   & \Gamma - \Lambda = 0 \\
   & \Xi - \Theta  = 0 \\
   & \Xi - \Omega = 0
   \end{aligned}\right. .
\end{align*}
It can be proved that if $\Gamma$ has the eigenvalues $\lambda_1,\ldots,\lambda_{t}$ and the corresponding eigenvectors $u_1,\ldots,u_{t}\in \mathbb{R}^{t}$, and $\Omega$ has the eigenvalues $\mu_1,\ldots,\mu_{s}$ and the corresponding eigenvectors $v_1,\ldots,v_{s}\in \mathbb{R}^{s}$, we will have
\begin{align*}
	&\frac{\mathrm{d} \log|\Omega \oplus \Gamma| }{\mathrm{d} \Gamma}   = \sum_{i=1}^{t} \left(\sum_{j=1}^{s}\frac{1}{\mu_j+\lambda_i}\right)u_i u_i^T,\\
&\frac{\mathrm{d} \log|\Omega \oplus \Gamma| }{\mathrm{d} \Omega} = \sum_{j=1}^{s} \left(\sum_{i=1}^{t}\frac{1}{\mu_j+\lambda_i}\right)v_j v_j^T.
\end{align*}
We terminate {\tt DNNLasso} when the relative KKT error is less than a given tolerance,  for example, $10^{-6}$.
Here  the relative KKT error refers to the degree to which the KKT optimality conditions are violated. It is a commonly used metric for assessing the accuracy of approximate solutions obtained from primal-dual methods. The relative KKT error refers to the maximum value of the following  quantities:
\begin{align}
&\frac{\|-\nabla \Gamma + R-X\|_F}{1+\|\nabla \Gamma\|_F + \|R\|_F + \|X\|_F}, \mbox{ where } \nabla \Gamma = \frac{\mathrm{d} \log |\Omega \oplus \Gamma|}{\mathrm{d} \Gamma},\label{eq1}\\
&\frac{\|-\nabla \Omega + U\|_F}{1+\|\nabla \Omega\|_F + \|U\|_F}, \mbox{ where } \nabla \Omega = \frac{\mathrm{d} \log |\Omega \oplus \Gamma|}{\mathrm{d} \Omega},\label{eq2}\\
&\frac{\|\Lambda - {\rm Prox}_p(\Lambda-X)\|_F}{1 + \|\Lambda\|_F + \|{\rm Prox}_p(\Lambda-X)\|_F},\label{eq3}\\
&\frac{\|\Theta - {\rm Prox}_q(\Theta-Y)\|_F}{1 + \|\Theta\|_F + \|{\rm Prox}_q(\Theta-Y)\|_F},\label{eq4}\\
&\frac{\| W- Y -U\|}{1+\|W\| + \|Y\| + \|U\|},\label{eq7}\\
&\frac{\|\Gamma - \Lambda\|_F}{1 + \|\Gamma\|_F + \|\Lambda\|_F},\label{eq5}\\
&\frac{\|\Xi - \Theta\|_F}{1 + \|\Xi\|_F + \|\Theta\|_F},\label{eq6}\\
&\frac{\|\Xi - \Omega\|_F}{1 + \|\Xi\|_F + \|\Omega\|_F}.\label{eq8}
\end{align}

These quantities include primal feasibility residuals (equations \eqref{eq5}, \eqref{eq6}, and \eqref{eq8}), dual feasibility residuals (equations \eqref{eq1}, \eqref{eq2}, and \eqref{eq7}), as well as complementarity slackness between primal and dual variables (equations \eqref{eq3} and \eqref{eq4}). Therefore, in our proposed algorithm, we use the relative KKT error to evaluate the optimality of the obtained approximate solutions.

For better illustration, we attach one example of the primal and dual residual plot on a Type 2 synthetic graph with $s =t = 500$ in Figure~\ref{fig-res-iter}. Moreover, we also plot the corresponding complementarity slackness between primal and dual variables.

\begin{figure}[H]
	\centering
	\includegraphics[width=0.34\textwidth]{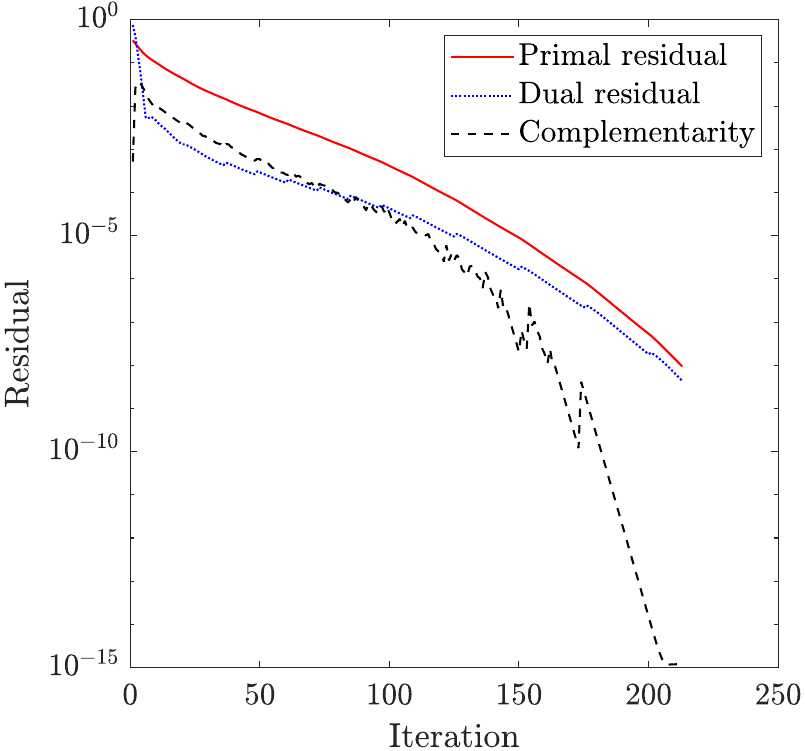}
	\caption{Relative KKT error decreasing on a {\bf Type 2} synthetic graph with $s=t=500$, $\lambda_0=10^{-2}$.}\label{fig-res-iter}
\end{figure}

\section{A SIMPLER VARIANT OF  DNNLASSO}\label{sec:admm-3block}
This section introduces a simpler variant of {\tt DNNLasso} (Algorithm~\ref{alg: admm-3d}) by eliminating the auxiliary variables $\Xi$. As we can see from \eqref{eq:KSv3}, $\Xi$ is a ``duplicate'' of the column-wise precision matrix $\Omega$, and at the optimal point they should be identical, i.e., $\Xi = \Omega$. Without $\Xi$,  this variant is simpler and has less variables compared with {\tt DNNLasso}.

\begin{algorithm}
	\caption{: A variant of {\tt DNNLasso}}
	\label{alg: admm-3d}
	\begin{algorithmic}
		\STATE {\bfseries Input:} Given sample covariance matrices $R\in\mathbb{S}^t_+$, $W\in \mathbb{S}^s_+$ and a parameter $\lambda_0>0$.
		\STATE {\bfseries Initialization:} Set $\lambda_T = \lambda_0 s$, $\lambda_S = \lambda_0 t$ and $\tau = 1.618$. Choose $\sigma > 0$. Choose an initial point $(\Omega^0,\Lambda^0,\Theta^0,X^0,Y^0)\in  \mathbb{S}^{s}\times\mathbb{S}^{t}\times \mathbb{S}^{s}\times \mathbb{S}^{t}\times \mathbb{S}^{s}$.  Set $k\leftarrow 0$.
		\REPEAT
		\STATE {\bfseries Step 1}. Compute
		\begin{align*}
			\Gamma^{k+1}=\Psi_{{\rm Right},1/\sigma,\Omega^k}(\Lambda^k +\frac{X^k}{\sigma} -\frac{R}{\sigma}),\ \
			\Omega^{k+1} =\Psi_{{\rm Left},1/\sigma,\Gamma^{k+1}}(\Theta^k +\frac{Y^k}{\sigma} -\frac{W}{\sigma}).
		\end{align*}
		\\
		\STATE {\bfseries Step 2}. Let $\widetilde{\Lambda}= \Gamma^{k+1} -X^k/\sigma$, $\widetilde{\Theta}= \Omega^{k+1} - Y^k/\sigma$. Compute $\Lambda^{k+1}\in \mathbb{S}^t$  and $\Theta^{k+1}\in \mathbb{S}^s$ as
	\begin{align*}
		\Lambda^{k+1}_{ij}   =   \left\{
		\begin{array}{ll}
		    \max(0, \widetilde{\Lambda}_{ii})   & \mbox{if }i = j\\
		    {\rm sgn}( \widetilde{\Lambda}_{ij} ) \max (| \widetilde{\Lambda}_{ij}|  -  \frac{\lambda_T}{\sigma},0) & \mbox{if }i \neq j
		\end{array}
		\right.    ,\ \
		\Theta^{k+1}_{ij}   =   \left\{
		\begin{array}{ll}
		   \max(0, \widetilde{\Theta}_{ii})  & \mbox{if }i = j\\
		    {\rm sgn}(  \widetilde{\Theta}_{ij} ) \max(| \widetilde{\Theta}_{ij}|  - \frac{\lambda_S}{\sigma},0) & \mbox{if }i \neq j
		\end{array}
		\right.    .
	\end{align*}
		\\[5pt]
		\STATE {\bfseries Step 3}. Update the multipliers by
		\begin{align*}
			X^{k+1}
			= X^k-\tau\sigma (\Gamma^{k+1} - \Lambda^{k+1}),\ \ Y^{k+1} = Y^k-\tau\sigma (\Omega^{k+1} - \Theta^{k+1}).
		\end{align*}
		\\
		\STATE {\bfseries Step 4}. Set $k\leftarrow k+1$.
		\UNTIL{Stopping criterion is satisfied.}
		
		\STATE {\bfseries Output:} An approximate solution $(\widehat{\Gamma}, \widehat{\Omega})$ 
computed as follows: $(\widehat{\Gamma}, \widehat{\Omega})=(\Gamma^{k},\Omega^{k})$ if $\Gamma^{k} \succ 0$, $\Omega^{k} \succ 0$; and $(\widehat{\Gamma}, \widehat{\Omega})=(\Gamma^{k}-  c I_t,\Omega^{k} + c I_s)$ with $c = (\lambda_{\rm min}(\Gamma^{k})-\lambda_{\rm min}(\Omega^{k}))/2$ otherwise.
	\end{algorithmic}
\end{algorithm}

\section{MORE NUMERICAL RESULTS ON LARGER SYNTHETIC DATA}
\label{sec:largedata}
To better demonstrate the superior performance of ${\tt DNNLasso}$, we show the comparison of {\tt DNNLasso}, ${\tt TeraLasso}$ and ${\tt EiGLasso}$ for learning the KS-structured precision matrices on larger synthetic data sets in the following Table \ref{table}. Specifically, we run our experiments on two types of graphs with dimensions $s=t=1500, 2000, 3000, 4000, 5000$. We set the maximum computational time of each method as 2 hours.

\begin{table}[H]
\centering
	\caption{Comparison of three methods on large synthetic data with $\lambda_0 = 0.01$.}\label{table}
	\renewcommand\arraystretch{1.2}
	\begin{tabular}{ |c|c|c|c|c|c|c|c|  }
		\hline
		&   &  \multicolumn{2}{|c|}{{\tt DNNLasso}}  & \multicolumn{2}{|c|}{{\tt TeraLasso}} & \multicolumn{2}{|c|}{{\tt EiGLasso}} \\
		\hline
		Graph &  $(s,t)$  &  Time (s) & Obj    & Time (s) & Obj &   Time (s) & Obj \\
		\hline
		\multirow{ 5}{*}{Type 1}
		& $(1500,1500)$  &  186 & -2.2596e6 & 7289 & -2.2322e6 & 7337 & -2.2576e6 \\
		\cline{2-8}
		&$(2000,2000)$  &  405 & -3.8564e6 & 7256 & -3.5893e6 & 7488 & -3.8375e6 \\
		\cline{2-8}
		&$(3000,3000)$  &  583 & -8.0733e6 & 7244 & -6.1875e6 & 7362 & -7.5311e6 \\
		\cline{2-8}
		&$(4000,4000)$  &  1549 & -1.3884e7 & -- & -- & -- & --\\
		\cline{2-8}
		&$(5000,5000)$  &  3386 & -2.1174e7 & -- & -- & -- & --\\
		\hline
		\multirow{ 5}{*}{Type 2}
		& $(1500,1500)$  &  240 &  -2.5098e6 & 7208 & -2.4604e6 & 7474 & -2.5064e6 \\
		\cline{2-8}
		&$(2000,2000)$  & 526  & -4.2782e6 & 7282 & -3.9906e6 & 7638 & -4.2661e6 \\
		\cline{2-8}
		&$(3000,3000)$  & 610  & -8.9069e6 & 7251 & -7.2176e6 & 7670 & -8.1600e6 \\
		\cline{2-8}
		&$(4000,4000)$  &  1558 & -1.5239e7 & -- & -- & -- & --\\
		\cline{2-8}
		&$(5000,5000)$  & 3338  & -2.3122e7 & -- & -- & -- & --\\
		\hline
	\end{tabular}
\end{table}

We can see from Table \ref{table} that, our proposed {\tt DNNLasso} performs better than the other two estimators by a large margin, in the sense that we take much less time but get much better objective function values. Moreover, we can see that even for the smallest data set with $s = t= 1500$, {\tt TeraLasso} and {\tt EiGLasso} can not achieve satisfactory perfermance within 2 hours, while our proposed {\tt DNNLasso} is able to solve the largest problem with $s=t=5000$ within one hour.

\section{MORE EXPERIMENTAL RESULTS ON COIL100 VIDEO DATA}
\label{sec:coil100}
We report more experimental results on COIL100 Video Data here for illustration. We pick another object (a cargo) from the data, which is illustrated in Figure \ref{fig:car}. From Figure \ref{fig:car}, we again find that the reduced resolution of $32\times 32$ is a good representation of the original $128\times 128$ resolution, while the reduced resolution of $8\times 8$ may not provide enough information. That is one evidence why we need an efficient and robust algorithm for estimating the large-scale KS-structured precision matrix otherwise it is impossible for us to deal with the case for $t=32\times 32=1024$ pixels within seconds.

\begin{figure}[H]
	\centering
	\includegraphics[width=0.4\textwidth]{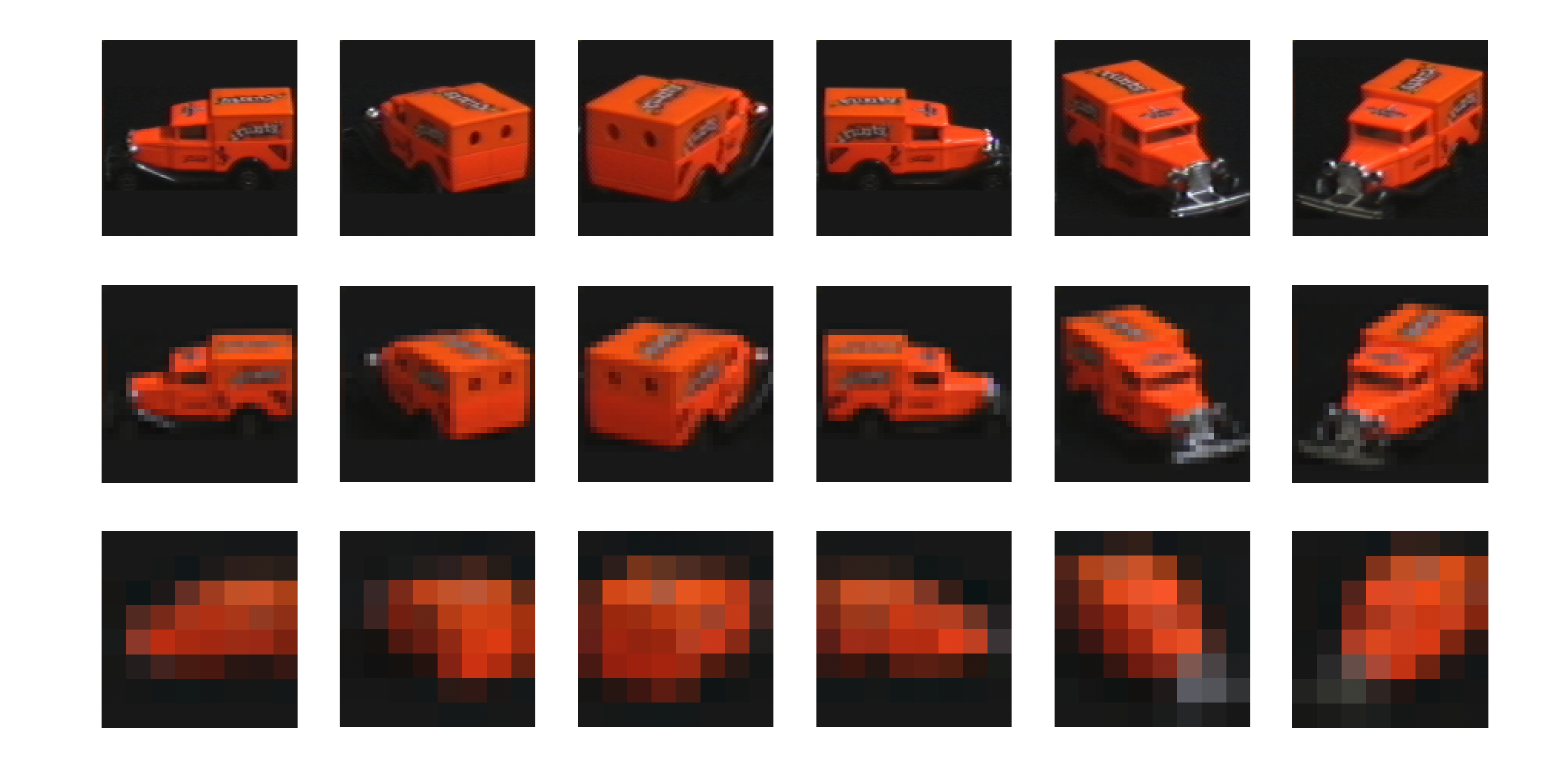}
	\caption{A rotating box of a cargo in COIL100 video data. First row: original resolution of $128\times 128$ pixels. Second (resp. third) row: reduced resolution of $32\times 32$ (resp. $8\times 8$) pixels.}
	\label{fig:car}
\end{figure}

\begin{figure}[H]
	\centering
	\subfigure[][]{\includegraphics[width=0.325\textwidth]{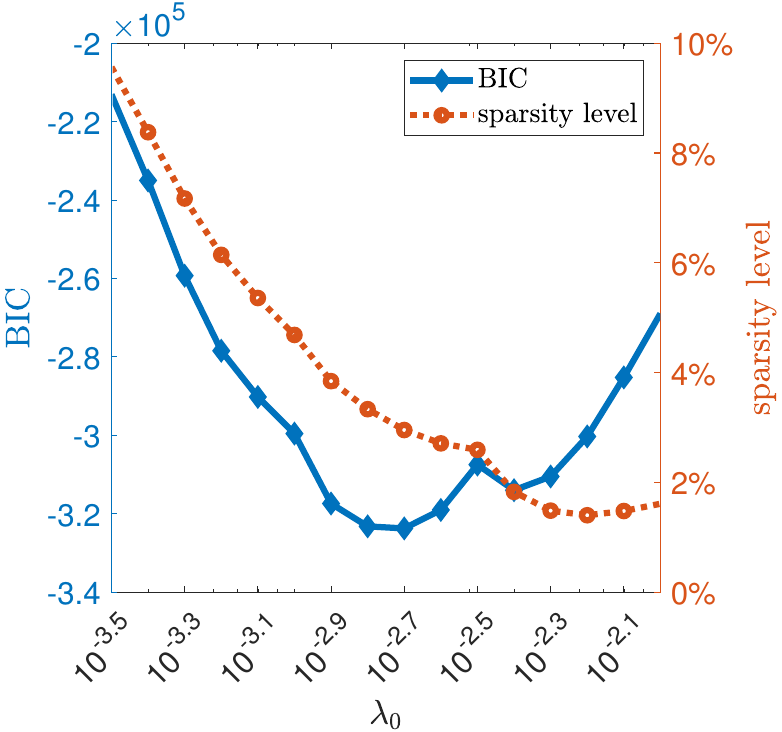}}\,\,
	\subfigure[][]{\includegraphics[width=0.3\textwidth]{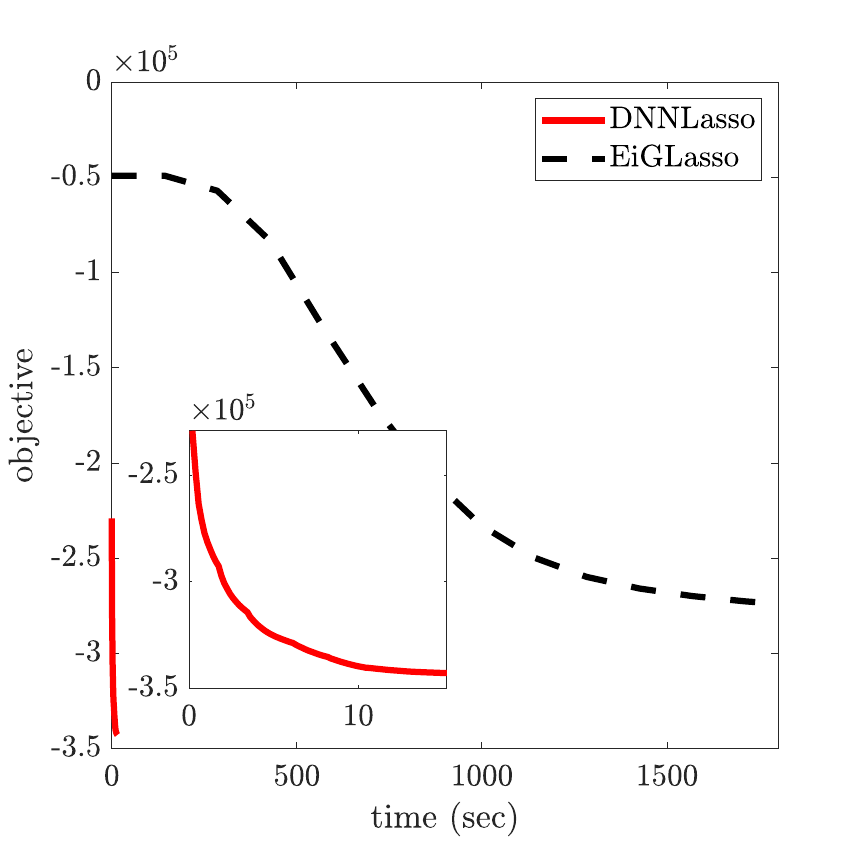}}
	\\
	\subfigure[][]{\includegraphics[width=0.26\textwidth]{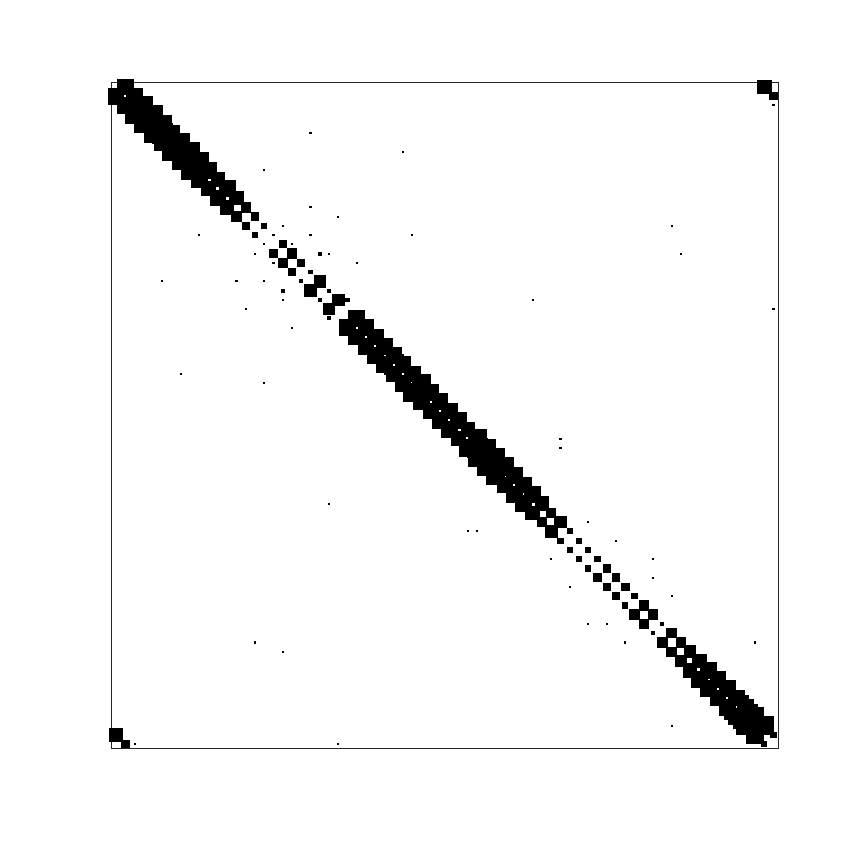}}\quad\quad
	\subfigure[][]{\includegraphics[width=0.3\textwidth]{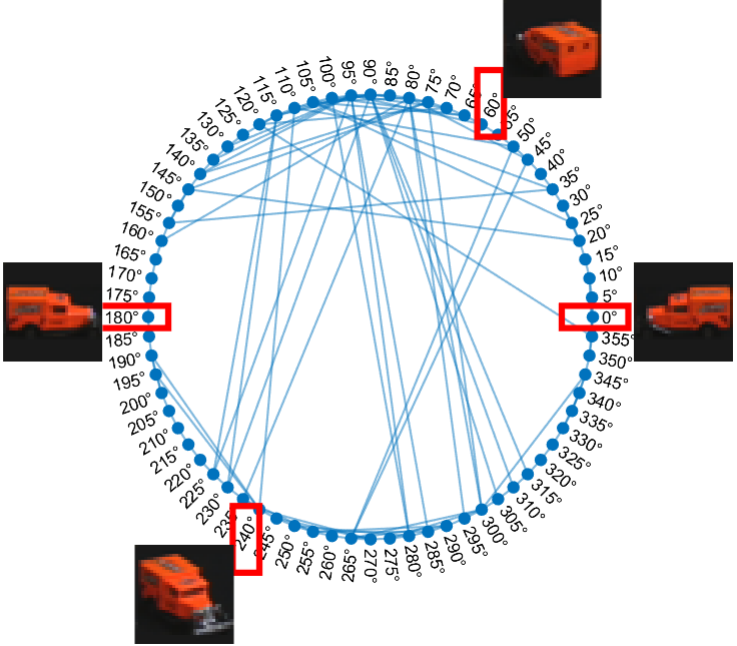}}
	\caption{On $s=72$ frames  with  $t=32\times 32$ pixels. (a) The BIC and sparsity level against $\lambda_0$. (b) The relative objective function value against computational time. 
		(c) Sparsity pattern of the  matrix $\widetilde{\Omega}\in\mathbb{S}^s$ estimated by {\tt DNNLasso} (i.e., the correlation pattern among frames from different angles).  (d) Relationship graph of frames from different angles.}
	\label{fig:coil100}
\end{figure}

We first conduct experiments on $s=72$ frames with the reduced resolution of $t=32\times 32$ pixels. We select parameter $\lambda_0$ from the set $\{10^{-3.5},10^{-3.4},\dots,10^{-2.1},10^{-2}\}$ under the Bayesian information criterion (BIC), and then compare our {\tt DNNLasso} with {\tt TeraLasso} and {\tt EiGLasso}. We terminate {\tt DNNLasso} with tolerance $5\times 10^{-3}$. Figure \ref{fig:coil100}(a) plots the BIC value and sparsity level against $\lambda_0$ and Figure \ref{fig:coil100}(b) illustrates the objective function value against computational time with $\lambda_0$ with the best BIC. {\tt TeraLasso} is not included in the figure as it failed to return a positive definite solution $\Omega\oplus \Gamma$. From Figure \ref{fig:coil100}(b) we can see that the objective function value obtained by {\tt DNNLasso} after 10 seconds is far much better than that obtained by {\tt EiGLasso} after more than 1600 seconds. In Figure \ref{fig:coil100}(c,d), we demonstrate the sparsity pattern of the  matrix $\widetilde{\Omega}\in\mathbb{S}^s$ estimated by {\tt DNNLasso}, namely the relationship graph of frames from different angles. The relationship graph in Figure \ref{fig:coil100}(c) indicates a manifold-like structure where image observed from $x^{\rm o}$ and $(x+360)^{\rm o}$ join, which is expected from a $360^{\rm o}$ rotation. The interesting different structure between Figure \ref{fig:coil100}(c) and Figure \ref{fig:coil1}(c) come from the natural observations from the objects: the box of cold medicine admits 180 degree symmetry while the cargo doesn't.

In addition, we also report the experimental results on the reduced resolution data with $t=8\times 8$ pixels in Figure \ref{fig:coil100_small}. We again find that there are some unexpected correlations among frames shown in Figure \ref{fig:coil100_small}(d), compared with Figure \ref{fig:coil100}(d), which implies that images of $8\times 8$ pixels are too blur to identify.

\begin{figure}[H]
	\centering
	\subfigure[][]{\includegraphics[width=0.325\textwidth]{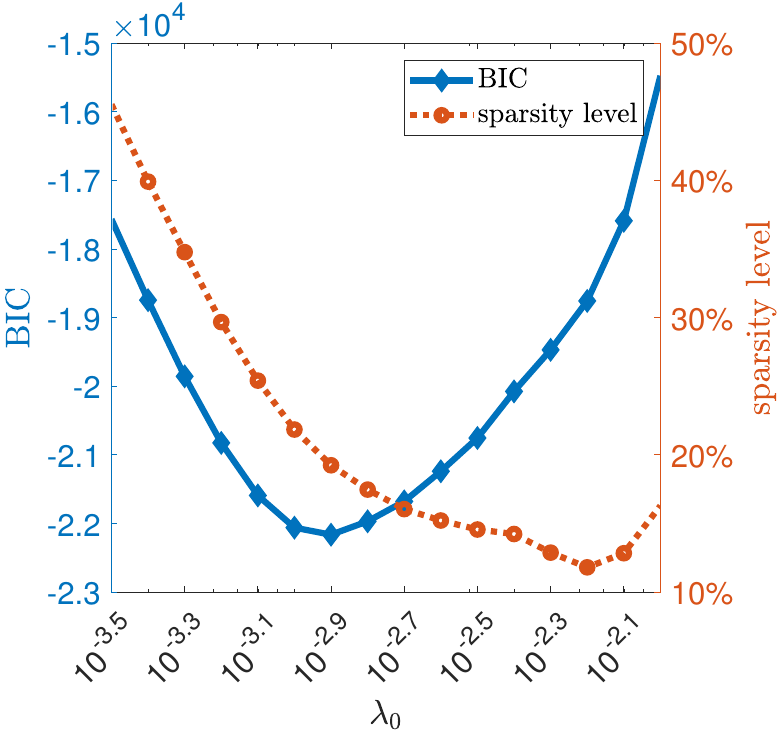}}\,\,
	\subfigure[][]{\includegraphics[width=0.3\textwidth]{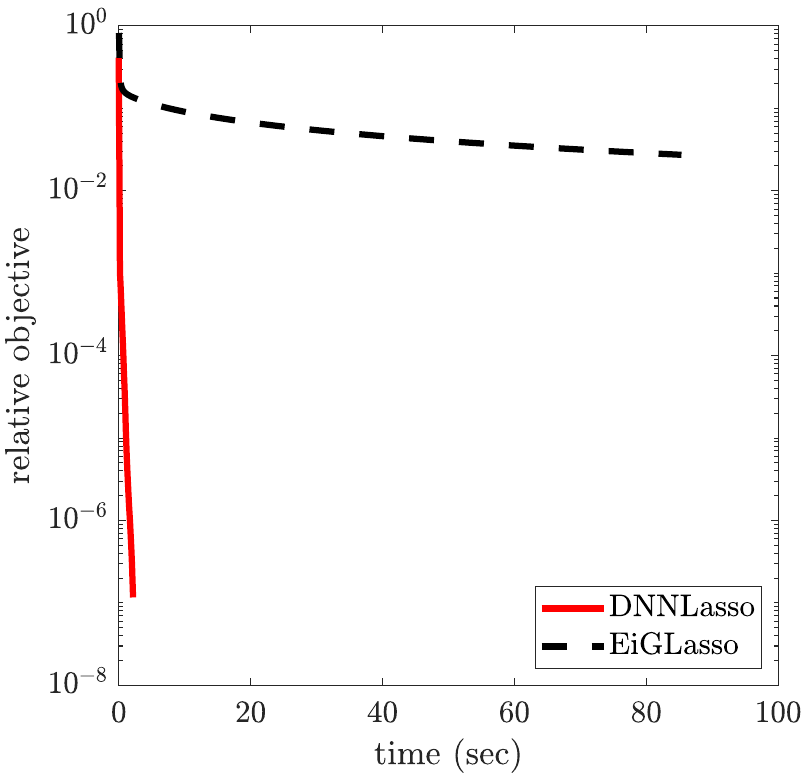}}
	\\ \subfigure[][]{\includegraphics[width=0.26\textwidth]{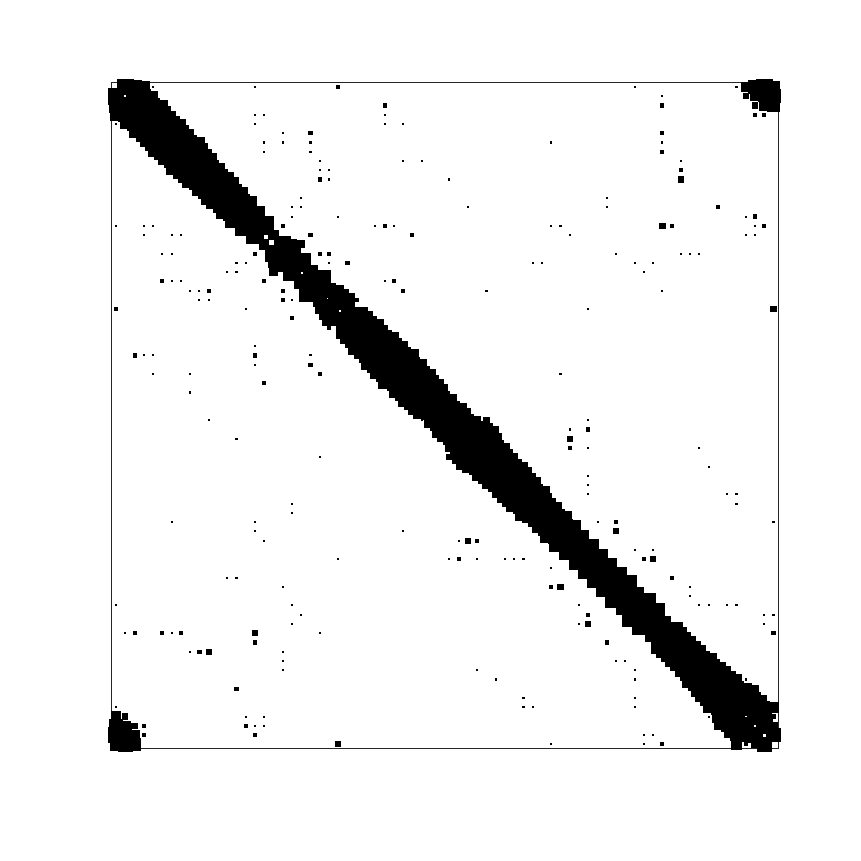}}\quad\quad
	\subfigure[][]{\includegraphics[width=0.3\textwidth]{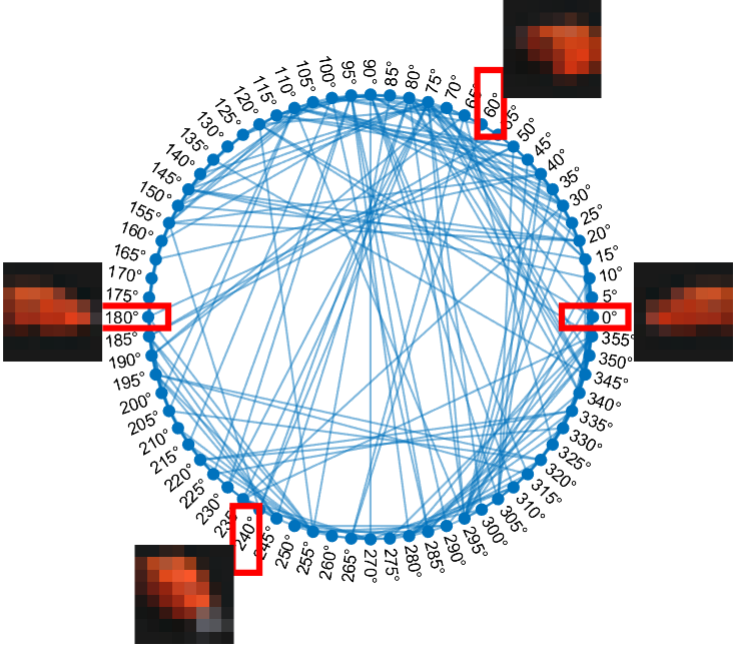}}
	\caption{On $s=72$ frames  with $t=8\times 8$ pixels. 
	}
	\label{fig:coil100_small}
\end{figure}

\end{document}